\documentclass[sigconf,nonacm]{acmart}
\makeatletter
\def\@ACM@checkaffil{
    \if@ACM@instpresent\else
    \ClassWarningNoLine{\@classname}{No institution present for an affiliation}%
    \fi
    \if@ACM@citypresent\else
    \ClassWarningNoLine{\@classname}{No city present for an affiliation}%
    \fi
    \if@ACM@countrypresent\else
        \ClassWarningNoLine{\@classname}{No country present for an affiliation}%
    \fi
}
\makeatother
\usepackage{xcolor}
\usepackage{colortbl}
\usepackage{nicematrix}
\usepackage{color} 
\usepackage{tabularx}
\usepackage{makecell}
\usepackage{multirow}
\usepackage{mwe,lipsum}
\usepackage{array,multirow}
\usepackage{float}
\usepackage{pifont}
\usepackage{arydshln}
\usepackage{graphicx}
\usepackage{multirow} 
\usepackage{caption}
\usepackage{subcaption}
\usepackage{makecell}
\usepackage{bbding}
\usepackage{tikz}
\definecolor{gold}{RGB}{221, 196, 65}
\definecolor{silver}{RGB}{215, 215, 215}
\definecolor{bronze}{RGB}{126, 66, 5}
\newcommand{\tikzcircle}[2][red,fill=red]{\tikz[baseline=-0.7ex]\draw[#1,radius=#2] (0,0) circle ;}
\usepackage{color,array}
\usepackage{xcolor}
\usepackage{colortbl}
\usepackage{url}
\usepackage{xcolor}

\definecolor{cvprblue}{rgb}{0.21,0.49,0.74}

\definecolor{ao}{rgb}{0.0, 0.5, 0.0}

\AtBeginDocument{%
  \providecommand\BibTeX{{%
    \normalfont B\kern-0.5em{\scshape i\kern-0.25em b}\kern-0.8em\TeX}}}

\setcopyright{acmcopyright}
\copyrightyear{2018}
\acmYear{2018}
\acmDOI{XXXXXXX.XXXXXXX}

\acmConference[Conference acronym 'XX]{Make sure to enter the correct
  conference title from your rights confirmation emai}{June 03--05,
  2018}{Woodstock, NY}
\acmPrice{15.00}
\acmISBN{978-1-4503-XXXX-X/18/06}

\begin{document}

\title{HOIN: High-Order Implicit Neural Representations}


\author{ Yang Chen \quad Ruituo Wu, \\ Yipeng Liu, \emph{Senior Member} \quad Ce Zhu, \emph{Fellow IEEE}}

\thanks{All the authors are with the School of Information and Communication Engineering, University of Electronic Science and Technology of China (UESTC), Chengdu, 611731, China. (e-mail: yangchen2023@std.uestc.edu.cn, wrt786842305@gmail.com, eczhu@uestc.edu.cn, yipengliu@uestc.edu.cn).}

\renewcommand{\shortauthors}{Wu and Wang, et al.}

\begin{abstract}
Implicit neural representations (INR)  suffer from worsening spectral bias, which results in overly smooth solutions to the inverse problem. To deal with this problem, we propose a universal framework for processing inverse problems called \textbf{High-Order Implicit Neural Representations (HOIN)}. By refining the traditional cascade structure to foster high-order interactions among features, HOIN enhances the model's expressive power and mitigates spectral bias through its neural tangent kernel's (NTK) strong diagonal properties, accelerating and optimizing inverse problem resolution. By analyzing the model's expression space, high-order derivatives, and the NTK matrix, we theoretically validate the feasibility of HOIN. HOIN realizes 1 to 3 dB improvements in most inverse problems, establishing a new state-of-the-art recovery quality and training efficiency, thus providing a new general paradigm for INR and paving the way for it to solve the inverse problem.
\end{abstract}

\keywords{Implicit Neural Representation, Inverse Problem, High-Order Feature Interaction.}

\begin{teaserfigure}
  \centering
  \includegraphics[width=\textwidth]{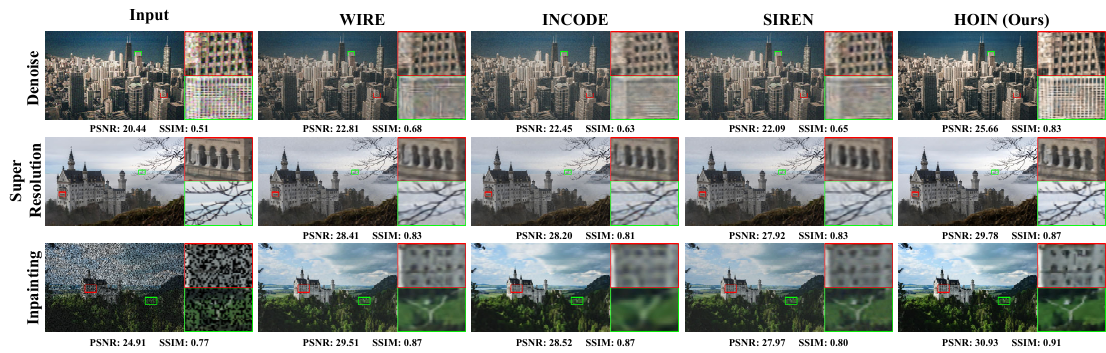}
  \caption{In this work, we propose a novel universal solution for inverse problems based on implicit neural representation (INR) - HOIN. Compared with traditional INR methods such as WIRE \cite{saragadam2023wire}, INCODE \cite{kazerouni2024incode}, and SIREN \cite{sitzmann2020implicit}, HOIN significantly improves the model's ability to perceive high-frequency information, effectively characterizes signal details, and achieves the best performance in a series of classic inverse tasks such as image denoise, super-resolution, and inpainting.}
  \Description{}
  \label{fig:teaser}
\end{teaserfigure}


\maketitle

\section{Introduction}

Convolutional Neural Networks (CNNs) effectively learn signals but struggle with high-frequency signals due to high impedance, i.e. spectral bias, which has become a significant challenge in signal processing. But deep image prior (DIP) \cite{ulyanov2018deep} takes advantage of spectral bias, successfully tackling image restoration tasks such as denoising, super-resolution, and other visual inverse challenges \cite{heckel2018deep,darestani2021accelerated,chakrabarty2019spectral}. DIP benefits from its independence from vast external data sets. However, its need for many training parameters and extended time frames limit practical use \cite{arican2022isnas,heckel2019denoising,ho2021neural}.

Implicit Neural Representation (INR) \cite{sitzmann2020implicit} refines signal modeling by integrating coordinate inputs with neural networks. By its structural advantages \cite{shabtay2022pip}, this method efficiently addresses inverse problems with reduced parameter count and processing time \cite{saragadam2023wire,kazerouni2024incode}. However, traditional INR methods can lead to worsening spectral bias, which results in overly smoothed solutions that omit vital high-frequency details \cite{saragadam2023wire}. New strategies have been introduced to solve this problem, such as adding an encoding layer that elevates coordinate inputs to higher-dimensional spaces \cite{tancik2020fourier,raghavan2023neural,singh2023polynomial,fathony2020multiplicative,mildenhall2021nerf} and utilizing periodic \cite{sitzmann2020implicit,liu2023finer,kazerouni2024incode} or non-periodic activation functions \cite{saragadam2023wire,ramasinghe2022gauss}. These modifications aim to fine-tune frequency responses automatically, mitigating the issue of spectral bias and enhancing the DIP process, thus getting more detailed outcomes.


However, the existing solutions still have challenges. On the one hand, They tend to be tailored for specific tasks \cite{saragadam2023wire}, needing more versatility for the broad spectrum of inverse problems. On the other hand, experimental evidence suggests that while these solutions may reduce spectral bias, they often fail to restore high-frequency details completely \cite{liu2023finer,kazerouni2024incode,saragadam2023wire}. During our experiments, as shown in Figure \ref{fig: Consolidated results}, we observed that hash coding \cite{muller2022instant,girish2023shacira}, despite its efficiency in eliminating spectral bias, inadvertently blends high-frequency noise with the signal when applied to inverse problems. This issue is incredibly challenging in tasks such as image denoising and deblurring. Therefore, there is an evident need for a solution universally applicable to all types of inverse problems and appropriately addresses spectral bias.

Incorporating high-order interaction structures \cite{wang2017sort,bu2021quadratic} into neural networks has dramatically expanded the hypothesis space, rapidly enhancing the ability to learn specific signal characteristics. This advancement is notably present in networks like Transformers \cite{vaswani2017attention} and Polynomial Neural Networks (PNNs) \cite{karras2019style,chrysos2021deep,xu2022quadralib}, which integrate multiplicative interactions and successfully address high-frequency signal processing challenges. Inspired by recent advancements, we diverge from traditional reliance on coding layers and activation functions, introducing an MLP block focused on higher-order feature interactions to present \textbf{High-Order Implicit Neural Representations for Inverse Problems (HOIN)}, a novel, generalized approach for tackling inverse problems through INRs. Our research shows that HOIN enhances the translational invariance and eigenvalue distribution in the Neural Tangent Kernels (NTK) \cite{jacot2018neural,choraria2022spectral} linked to INRs, thereby expanding the functional space of the model. HOIN significantly improves its capacity to mitigate spectral bias, excels at modeling high-frequency signals, and effectively minimizes noise interference.

To conclude, our contributions can be summarized as follows:

\begin{itemize}
    \item We propose a new high-order interaction block to mitigate the worsening spectral bias in INR.
    \item We propose a universal inverse problem-handling framework, the HOIN, that can apply INR to any inverse problem.
    \item We analyze the expression ability, higher-order derivatives, and NTK matrices of the higher-order blocks and theoretically prove the higher-order blocks' effectiveness.
    \item HOIN maintains the state-of-the-art (SOTA) performance in various models that use INR to solve inverse problems and representation tasks.
\end{itemize}


\begin{figure*}[!t]
	\centering
 	\includegraphics[width=1\textwidth,keepaspectratio]{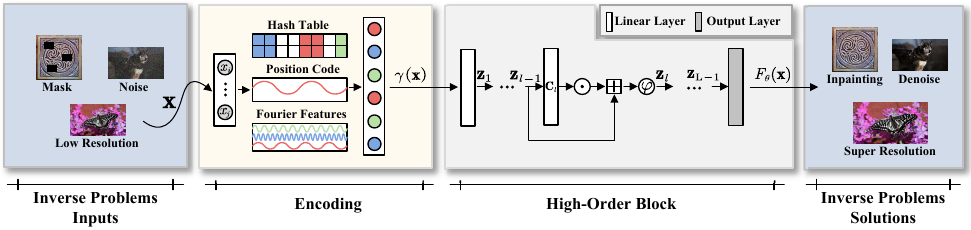}
	\caption{Overview of HOIN. We select the corresponding encoding layer based on the type of inverse problem, mapping the coordinate input $\mathbf{x}$ into a higher dimensional space $\gamma(\mathbf{x})$. Then, the low-frequency and high-frequency information in the signal is captured through a High-Order Block structure. During training, we find the peak performance point, stop the fitting process there, and find the solution $ F_{\mathbf{\theta }}(\mathbf{x})$. $\odot$ denotes Hadamard product, $\boxplus $ is addition, $\varphi$ is the nonlinear activation function.}
	\label{fig:Overview}
 \Description{}
\end{figure*}

\section{Background}

\subsection{Inverse Problem}
Solving inverse problems, which aim to reconstruct original signals from measurements, is crucial in critical applications like image restoration and sound source localization. Traditional methods for these problems often rely on existing knowledge, proposing solutions that meet certain conditions or combine an understanding of the target's structure with sparsity assumptions \cite{tibshirani1996regression, chambolle2004algorithm, baraniuk2010model, romano2017little}. However, These approaches face challenges in more complex situations. Deep learning methods introduce innovative solutions like deep image priors (DIP) \cite{ulyanov2018deep} and implicit neural representations (INR) \cite{sitzmann2020implicit}. Leveraging spectral deviation priors, INR can rapidly address and derive solutions for inverse problems from a single sample. This approach demands fewer parameters than methods based on convolutional networks and markedly abbreviates the training duration.

\subsection{Implicit Neural Representation Details}

INR \cite{sitzmann2020implicit} uses coordinate grids to approximate continuous signals, showing advantages in rendering, computational imaging, medical imaging, and virtual reality over traditional methods \cite{kuznetsov2021neumip,mildenhall2021nerf,chen2021learning}. Recently, INR's approach to solving inverse problems has gained notable attention \cite{saragadam2023wire}.

In an inverse problem, suppose coordinate inputs $\mathbf{x}\in \mathbb{R}^{D_i}$ corresponding to the clean signals $ S(\mathbf{x}):\mathbb{R}^{D_i} \mapsto \mathbb{R}^{D_o}$ and the noise signal $ N(\mathbf{x}):\mathbb{R}^{D_i} \mapsto \mathbb{R}^{D_o}$. For image, we have the coordinate input $(x_i,x_j)$, and the corresponding image $ S(\mathbf{x}) \in \mathbb{R}^{3\times H \times W}$. The noise signal can be modeled as
\begin{equation}
    N\left( \mathbf{x} \right) =S\left( \mathbf{x} \right) +\mathbf{n},
\end{equation}
where $\mathbf{n}$ is assumed to be $i.i.d.$ Gaussian Noise drawn from $\mathcal{N}\left(0, \sigma^2 \mathbf{I}\right)$ with $\mathbf{I}$ being the identity matrix.

INR parameterizes the clean signal $S(\mathbf{x})$ via a network $ F_{\mathbf{\theta }}(\mathbf{x}):\mathbb{R}^{D_i} \mapsto \mathbb{R}^{D_o}$ and is optimized to fit the noisy signal $ N(\mathbf{x})$, formulated as:

\begin{equation}
\theta ^*=\underset{\theta}{\mathrm{arg}\min}\mathcal{L} \left( N\left( \mathbf{x} \right) ;F_{\theta}(\mathbf{x}) \right) ,\quad S^*\left( \mathbf{x} \right) =F_{\theta ^*}(\mathbf{x}).
\end{equation}




Such parameterization allows lower-frequency contents to be fitted before the higher-frequency ones, exhibiting high impedance to signal noises or degradations. In practice, $\theta$ is usually learned using an MLP, and the overall network architecture of INR is as follows:

\begin{equation}
\begin{aligned}
 \mathbf{z}_0=&\gamma (\mathbf{x}), \\
 \mathbf{z}_l=&\varphi(\mathbf{C}_l \mathbf{z}_{l-1}) \\
 =&\varphi\left(\mathbf{W}_l \mathbf{z}_{l-1}+\mathbf{b}_l\right), l=1,2, \ldots, L-1, \\
 F_{\mathbf{\theta }}(\mathbf{x})=&\mathbf{W}_L \mathbf{z}_{L-1}+\mathbf{b}_L,
\end{aligned}
\end{equation}
where $\displaystyle \mathbf{z}^{\:l}$ denotes the output of layer $l$, $\theta=\{W^{l}, \mathbf{b}^{\:l}\ |\ l=1,2,...,L-1\}$, $L$ is the number of layers, $\varphi$ is the nonlinear activation function, $\gamma (\cdot )$ is the coding layer. $\mathbf{C}_l$ is linear function with respect to  $\mathbf{z}_{l-1}$.

\subsection{Motivation}

In traditional approaches, INR can tackle inverse problems but is constrained by worsening spectral bias. This worsening spectral bias typically leads to excessively smooth solutions that lack crucial high-frequency details. Methods such as nonlinear activation functions and high-dimensional encoding have been implemented to mitigate this issue, but their effectiveness across a wide range of inverse problems is limited. Acknowledging these challenges, our goal is to conduct an in-depth analysis of the root causes of spectral bias and devise a solution strategy that is more universally applicable.

\section{HOIN: High-Order Implicit Neural Representations}


\subsection{Overview}


In this Section, we introduce High-Order Implicit Neural Representations for Inverse Problems (HOIN). As illustrated in Figure \ref{fig:Overview}, the HOIN framework is through the coding layer and the high-order interaction block stages. 1) The coding layer transforms the coordinate input of signals (e.g. audio, image, video, etc.) $\mathbf{x}$ into a high-dimensional space (Section \ref{sec:Encoding Layer}); 2) By utilizing various activation functions, the high-order interaction block facilitates complex interactions among features within this expanded space (Section \ref{sec:High-Order Interaction Block}). During training, we find the peak performance point and stop the fitting process there.



\subsection{Encoding Layer}\label{sec:Encoding Layer}


Using the encoding layer first in HOIN aims to alleviate spectral bias by mapping the signal coordinates to high-dimensional space, enhancing the model's ability to capture details \cite{tancik2020fourier,mildenhall2021nerf}. INR alleviates spectral bias by mapping the signal coordinates to high-dimensional space, enhancing the model's ability to capture details \cite{tancik2020fourier,mildenhall2021nerf}. In addressing inverse problems, the deployment of coding layers has become crucial. Essential coding methods include positional coding (Pos. Enc), Fourier features (FFN), and hash table mapping (InstantNGP). In our proposed HOIN framework, as shown in Figure \ref{fig:Overview}, we adopt the following specific encoding strategies based on different types of inverse problems:

\begin{itemize}
\item Hash Table \cite{muller2022instant}: 
\begin{equation}
\gamma(\mathbf{x})=\left(\bigoplus_{d=1}^{D_i} x_d \pi_d\right) \quad \bmod T,
\end{equation}
\item Position Encoding \cite{mildenhall2021nerf}: 
\begin{equation}
\gamma (\mathbf{x})=\left[\cos \left( 2\pi \sigma ^{j/m}\mathbf{x} \right) ,\sin \left( 2\pi \sigma ^{j/m}\mathbf{x} \right)\right] ^{\mathrm{T}}\, j=0, \ldots, m,
\end{equation}
\item Fourier Features \cite{tancik2020fourier}:
\begin{equation}
 \gamma(\mathbf{x})=[\cos (2 \pi \mathbf{B x}), \sin (2 \pi \mathbf{B x})]^{\mathrm{T}},
 \end{equation}
\end{itemize}
where $\oplus$ denotes the bit-wise XOR operation and $\pi_d$ are unique large prime numbers. $T$ is the size of the hash table. each entry in $\mathbf{B} \in \mathbb{R}^{m \times {D_i}}$ is sampled from $\mathcal{N}\left(0, \sigma^2\right)$, $m$ is the mapping size, and $\sigma$ is chosen for each task and dataset with a hyperparameter sweep. 

\subsection{High-Order Interaction Block}\label{sec:High-Order Interaction Block}

\subsubsection{\textbf{Rethinking Plain Block and Residual Block}}


To address the worsening spectral bias in inverse problems, past enhancements have mainly concentrated on refining the coding layer and activation function, overlooking the crucial role of the MLP architecture within the INR. This oversight leaves the cascade architecture of the MLP unexamined, which is instrumental in the root cause of the worsening spectral bias \cite{liu2023finer,yuce2022structured}. 

For the classic INR model, the Plain MLP Block is

\begin{itemize}
    \item Plain Block \cite{rahaman2019spectral}:
    \begin{equation}
    \mathbf{z}_l=\varphi(\mathbf{C}_l\mathbf{z}_{l-1}). 
    \end{equation}
\end{itemize}


The cascade effect observed in plain blocks is the primary cause of spectral bias in INR. This problem presents itself in two significant ways: Firstly, with an increase in the number of block layers, the vanishing gradient issue becomes more pronounced, making the training process more challenging \cite{rahaman2019spectral}. Secondly, using the ReLU activation function can lead to the loss of high-order signal derivatives, further intensifying spectral bias \cite{he2016deep}. 

Residual blocks featuring residual connections have been a method to improve gradient flow to deeper layers. The expression is as follows
\begin{itemize}
    \item Residual Block \cite{he2016deep}:
    \begin{equation}
    \mathbf{z}_l=\varphi(\left( \mathbf{I}+\mathbf{C}_l \right) \mathbf{z}_{l-1}).
    \end{equation}
\end{itemize}

However, the residual block continues to face challenges with worsening spectral bias and learning high-frequency information \cite{belfer2021spectral}. It has yet to improve the efficiency of processing inverse problems significantly.

\begin{figure*}[!t]
    \centering
    \subfloat[]
    {\includegraphics[width=0.72\textwidth]{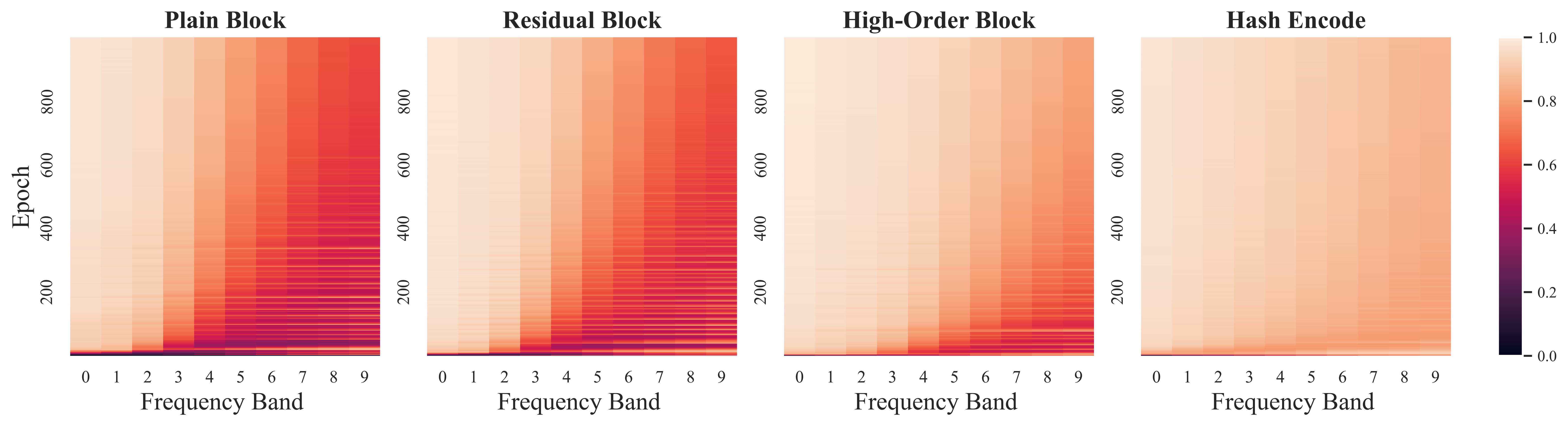}}
    \subfloat[]
    {\includegraphics[width=0.28\textwidth]{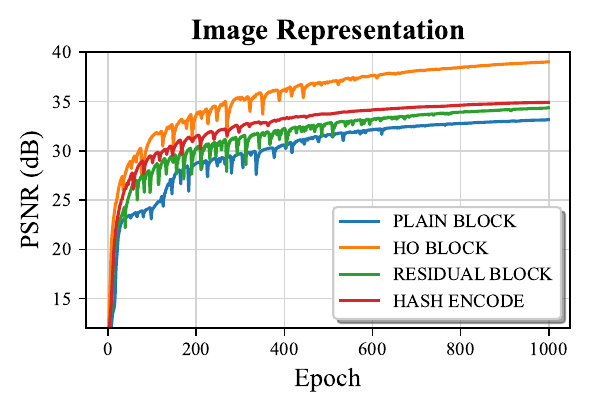}}
\caption{(a) Comparison of learning speeds at different frequencies. The target image is transformed into 10 frequency bands through the Fourier transform (x-axis, 0 represents the lowest frequency band), and we compare the learned components with the proper amplitude. On the color chart scale, 1 represents a perfect approximation. HO block can effectively alleviate spectral bias. Hash encoding does not exhibit spectral bias. (b) PSNR learning curves for different blocks. HO block maintains the highest PSNR.}
\label{fig: Consolidated results}
\Description{}
\end{figure*}

\subsubsection{\textbf{HO Block}}

Inspired by high-order interactions\cite{rao2022hornet,chrysos2023regularization,fan2023expressivity} in neural networks, we introduce a novel element into the MLP architecture of INR, which is called the \textbf{High-Order (HO) Block}. This addition aims to facilitate complex feature interactions at a higher level than traditional methods. This structure can be expressed in the following form:

\begin{equation}
\mathbf{z}_l=\varphi(\left( \mathbf{J}+\mathbf{C}_l\mathbf{z}_{l-1} \right) \odot\mathbf{z}_{l-1}),
\end{equation}
where $\odot$ denotes Hadamard product, $\mathbf{J}$ is the all-one matrix

We incorporate second-order interaction within the plain block by multiplying the previous layer's outputs with those of the current layer and then summing them up. This approach evolves into creating HO blocks through a hierarchical linkage, empowering the model to facilitate $2^{(L-1)/2}$-order feature interactions. Such an augmentation in high-order interactions diminishes the model's reliance on low-frequency learning. With increasing model depth, high-order blocks are designed to avoid the issue of gradient vanishing, enabling effective fitting of both high and low frequencies in the initial stages of training. This capacity allows for a swift alignment with the objective function of the real signal $G(\mathbf{x})$.


The HOIN framework tailors its approach to various inverse problems by selecting suitable encoding layers and activation functions to meet the specific demands of each task. An overly aggressive correction for spectral bias and the rapid acceleration of high-frequency learning might inadvertently blend noise with the signal, often detrimentally affecting the task. We introduced HO blocks to SIREN \cite{sitzmann2020implicit}, Pos.Enc \cite{mildenhall2021nerf}, and FFN \cite{tancik2020fourier}, creating \textbf{HO-SIREN}, \textbf{HO-Pos.Enc}, and \textbf{HO-FFN}, respectively. For particular inverse problem scenarios, we evaluate these models to identify the most effective one for deployment.

\section{Theoretical Analysis of HOIN}

In this section, we perform a theoretical analysis of HO blocks. In Sections \ref{sec:Expression ability}, \ref{sec:High-Order Derivative}, and \ref{sec:Neural Tangent Kernel}, we analyze the expressive ability, high-order derivatives, and NTK properties of various blocks.

\begin{figure*}[!t]
    \centering
    \subfloat[]
    {\includegraphics[width=0.68\textwidth]{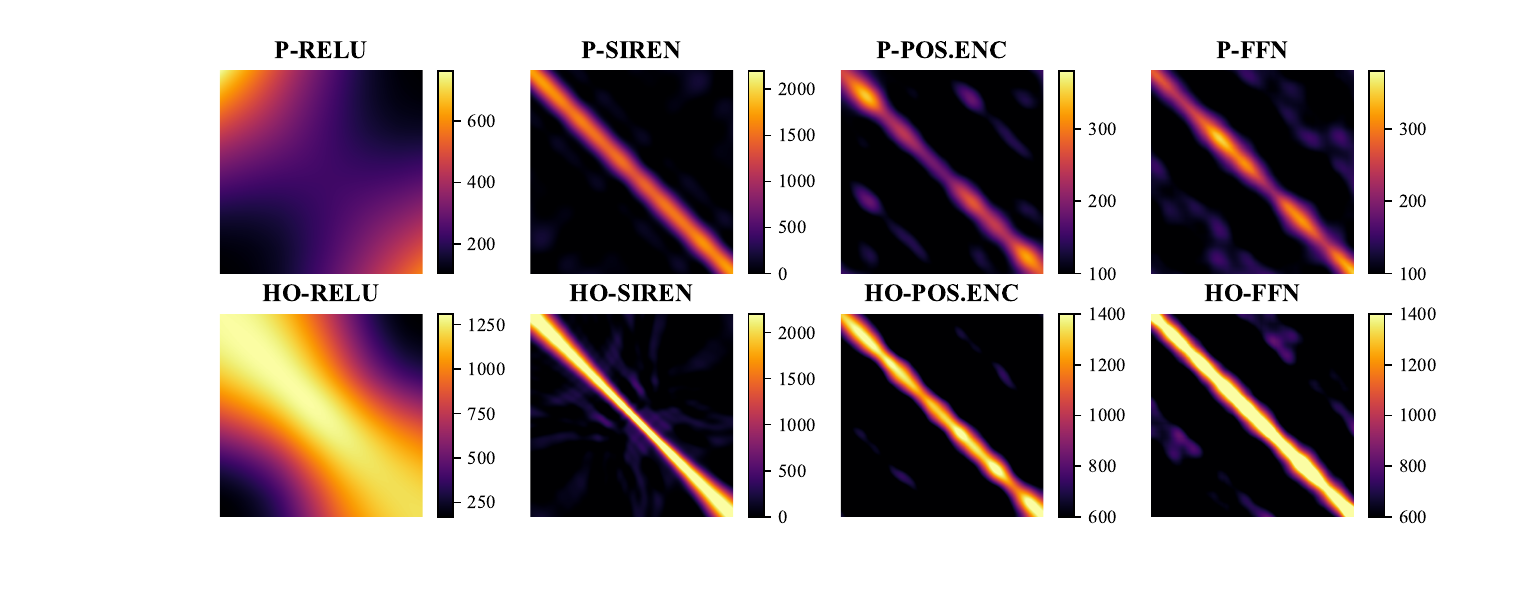}}
    \subfloat[]
    {\includegraphics[width=0.32\textwidth]{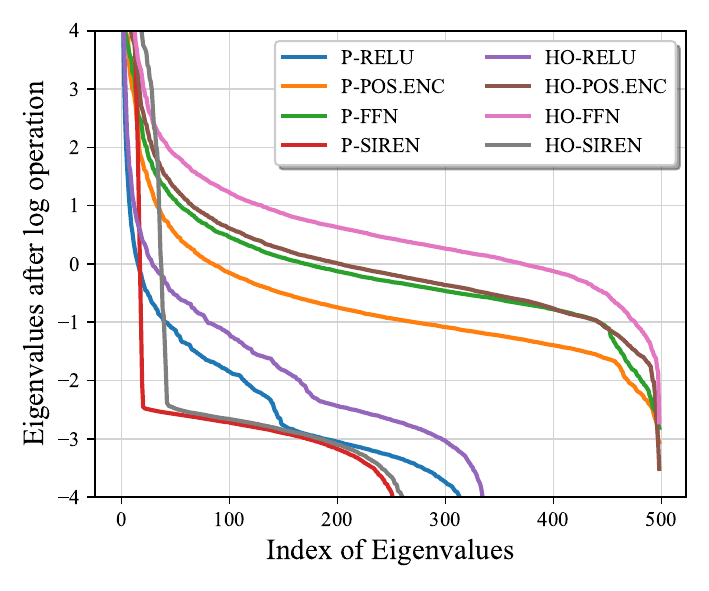}}
\caption{(a) Visualization of NTK and corresponding eigenvalues in different models. (b) Draw the corresponding feature values. Because the maximum eigenvalue is much larger than the minimum eigenvalue, all eigenvalues are processed by logarithmic functions for visualization. HO blocks significantly enhance the eigenvalues on the diagonal of the NTK matrix, thus enhancing the ability of the INR to capture high-frequency information. Plain, residual, and high-order blocks are abbreviated as P, R, and HO. }
\Description{}
\label{fig: ntk}
\end{figure*}

\subsection{Expression Ability Exploration }\label{sec:Expression ability}

In INR frameworks, the dimension of a network's functional space is a crucial metric for assessing the network's capacity for expression \cite{bu2021quadratic}. The architecture of the network is denoted by $\mathcal{D} =\left\{ D_1,...,D_l \right\}$, where $D_l$ indicates the number of neurons in the $l$-th layer. Any given activation function block can be decomposed into a series of polynomial functions with leading degree $r$ through Taylor approximation. This process helps understand how activation functions and network configurations influence INR models' functional capacity and expressiveness. 

For the network architecture $D_l$ with an activation function of leading degree $r$, we represent the leading functional space of the neural network as $\mathcal{F} _{\mathcal{D},r}$. The leading functional variants of plain Block, residual block, and HO block can be defined as the $Zariski$ $closure$ \cite{kileel2019expressive} of their leading functional space, i.e. $\mathcal{V} _{\mathcal{D},r}^{P}$,$\mathcal{V} _{\mathcal{D},r}^{R}$ and $\mathcal{V} _{\mathcal{D},r}^{HO}$ (similar to the ones presented in \cite{bu2021quadratic,kileel2019expressive}). Using these definitions, we have:

\begin{theorem}
\label{Theorem:efficiency}
For an activation function with leading degree $r \geq 1$ and network architecture $\mathcal{D} =\left\{ D_1,...,D_l \right\}$, the leading functional variety of Plain Block, $\mathcal{V} _{\mathcal{D} ,r}^{P}$, HO Block, $\mathcal{V} _{\mathcal{D} ,r}^{HO}$, and Residual Block, $\mathcal{V} _{\mathcal{D} ,r}^{R}$, satisfy:
\begin{equation}
    \mathcal{V} _{\mathcal{D} ,r}^{HO} = \mathcal{V} _{\mathcal{D} ,2r}^{R} = \mathcal{V} _{\mathcal{D} ,2r}^{P}.
\end{equation}

Proof. See Supplementary
\end{theorem}


Theorem \ref{Theorem:efficiency} posits that within an identical network architecture, HO block with a leading degree of $r$ and plain and residual block with a leading degree of $2r$ exhibit the same variation in homogeneous polynomial functions. This implies that HO block can support a diversity of leading functions, specifically $(2r)^{l-1}$ homogeneous polynomials, in contrast to neural networks of the same structure and activation function, which are limited to $r^{l-1}$. Consequently, the HO block possesses a more expansive expression space, representing a broader range of frequency signal components.

Besides, we give the frequency decay rates for different blocks as follows:
\begin{theorem} 
\label{Theorem:frequency decay}
For a single-layer MLP, the number of neurons is $d$. For frequency $k$, the frequency decay rate of the Plain and Residual Block is $k^{-d}$, while on the contrary, the frequency decay rate of the HO Block is $k^{{-d}/2}$.

Proof. See Supplementary
\end{theorem}

Theorem \ref{Theorem:frequency decay} suggests that HO blocks can capture more high-frequency information than plain blocks, facilitating a faster resolution of inverse problems.

We perform an image representation experiment to verify the phenomenon of spectral bias for different models. We employ an MLP model with three hidden layers, utilizing a ReLU activation function and positional encoding, to train on authentic natural images from the DIV2K \cite{timofte2018ntire} dataset. Our evaluation centers on the model's capacity to learn information across different frequency bands. Following the methodology outlined in \cite{shi2022measuring}, we partition the image spectrum into ten frequency bands and monitor the model's learning progress on these frequency bands during the training process. The experimental results are depicted in Figure \ref{fig: Consolidated results}, where darker red colors indicate weaker learning abilities of the model for those frequency bands. The findings suggest that the plain block struggles to learn high-frequency information in the image, even with positional encoding. Conversely, the HO block captures high-frequency features early in training, showcasing fast learning rates and excellent representation ability.

Additionally, upon introducing mainstream hash encoding representations, we observe its capability to learn balanced high-frequency and low-frequency information, exhibiting no spectral bias. Hash coding, a lattice-based interpolation method, learns low and high-frequency information. However, in inverse tasks, this leads to the blending of signal noise and high-frequency details, causing noise to be fitted early in the training, which is undesirable for inverse tasks. Specific experiments on this are conducted in the supplementary material.

\subsection{Derivative Analysis}\label{sec:High-Order Derivative}

In signal processing, first and second derivatives are pivotal for encapsulating rich high-frequency information. Spectral bias often arises because high-order derivatives gravitate towards zero during the learning process with plain and residual blocks, leading to a substantial loss of high-frequency details. The HO block is ingeniously designed to mitigate this issue. Specifically, the first derivative of the HO block $\nabla \mathbf{z}_l$ is
\begin{equation}
\nabla \mathbf{z}_l=\mathbf{C}_l\odot \mathbf{z}_{l-1}+\mathbf{C}_l\mathbf{z}_{l-1}+\mathbf{J}.
\end{equation}

The second derivative of the HO Block $\Delta \mathbf{z}_l$ is
\begin{equation}
\Delta \mathbf{z}_l=2\mathbf{C}_l.
\end{equation}

Furthermore, we compute the first-order derivatives of the plain and residual blocks, with the findings presented in Table \ref{tab: High-Order Derivative}. According to Table \ref{tab: High-Order Derivative}, the second derivative of plain blocks equals zero, indicating a significant loss of signal detail during processing. In contrast, the second derivative of the HO block remains constant, a property that significantly enhances its ability to capture detailed signal information. This characteristic of the HO block is instrumental in efficiently accelerating the resolution of inverse problems by preserving and leveraging high-frequency details often lost in traditional processing blocks.

\begin{table}[!t]
    \caption{First and second-order derivatives of different blocks. Plain, residual, and high-order blocks are abbreviated as P, R, and HO. $\mathbf{O}$ is an all-zero matrix.}
    \centering
    \begin{tabular}{ccc}
    \hline 
    \hline
        Block & $\nabla \mathbf{z}_l$ & $\Delta \mathbf{z}_l$\\ 
        \hline
        {P}     & $\mathbf{C}_l$ & $\mathbf{O}$\\
        {R}    & $\mathbf{C}_l+\mathbf{I}$ & $\mathbf{O}$\\
        {HO}  & $\mathbf{C}_l\odot \mathbf{z}_{l-1}+\mathbf{C}_l\mathbf{z}_{l-1}+\mathbf{J}$&$2\mathbf{C}_l$ \\
        \hline
        \hline
    \end{tabular}

    \label{tab: High-Order Derivative}
\end{table}

\begin{table*}[!t]
\caption{Results of image representation of different downsampling factors. Best 3 scores in each metric are marked with gold \tikzcircle[gold,fill=gold]{2pt}, silver \tikzcircle[silver,fill=silver]{2pt} and bronze \tikzcircle[bronze,fill=bronze]{2pt}.}
	\resizebox{\linewidth}{!}{
\begin{tabular}{lllllll|lllllll}
\hline\hline
\multicolumn{1}{l}{\multirow{2}{*}{Methods}} & \multicolumn{1}{l}{} & \multicolumn{2}{c}{$8\times$} & \multicolumn{3}{c}{$4\times$} & \multicolumn{1}{l}{} & \multicolumn{3}{c}{$2\times$}  & \multicolumn{3}{c}{$1\times$}  \\
\cmidrule(lr){3-4}\cmidrule(lr){5-7}\cmidrule(lr){9-11}\cmidrule(lr){12-14}
\multicolumn{1}{c}{}                         & \#Param $\downarrow$                & PSNR $\uparrow$      & SSIM $\uparrow$     & PSNR $\uparrow$  & SSIM $\uparrow$ & LPIPS $\downarrow$& \#Param $\downarrow$               & PSNR $\uparrow$  & SSIM $\uparrow$ & LPIPS $\downarrow$& PSNR $\uparrow$  & SSIM $\uparrow$ & LPIPS $\downarrow$\\
\hline
InstantNGP \cite{muller2022instant}                                    & 0.233                & 54.081     & 0.992     & 36.928 & 0.938 & 0.065 & 1.588                & 41.864 & 0.960 & 0.012 \tikzcircle[bronze,fill=bronze]{2pt} & 36.112 & 0.856 & 0.116 \\
WIRE \cite{saragadam2023wire}                                         & 0.437                & 54.392 \tikzcircle[bronze,fill=bronze]{2pt}     & 0.996 \tikzcircle[bronze,fill=bronze]{2pt}     & 38.809 & 0.947 & 0.036 & 0.889                & 42.335 & 0.963 & 0.012 \tikzcircle[silver,fill=silver]{2pt} & 35.037 & 0.895 & 0.111 \\
INCODE \cite{kazerouni2024incode}                                       & 0.207                & 50.697     & 0.989     & 39.236 & 0.918 & 0.057 & 1.029                & 42.771 \tikzcircle[bronze,fill=bronze]{2pt} & 0.967 & 0.013 & 36.723 & 0.882 & 0.108 \tikzcircle[bronze,fill=bronze]{2pt} \\
SIREN \cite{sitzmann2020implicit}                                        & 0.199 \tikzcircle[silver,fill=silver]{2pt}                 & 51.651     & 0.995     & 35.633 & 0.924 & 0.055 & 0.791 \tikzcircle[gold,fill=gold]{2pt}                & 41.478 & 0.965 & 0.015 & 33.505 & 0.872 & 0.155 \\
Pos. Enc \cite{mildenhall2021nerf}                                    & 0.204                & 31.016     & 0.901     & 31.322 & 0.872 & 0.136 & 0.805                & 36.789 & 0.939 & 0.046 & 32.954 & 0.884 & 0.159 \\
FFN \cite{tancik2020fourier}                                          & 0.329                & 46.031     & 0.986     & 38.609 & 0.960 & 0.047 & 1.314                & 41.801 & 0.973 & 0.013 & 34.206 & 0.916 & 0.122 \\
\hline
                                             \multicolumn{14}{c}{\textbf{Ours}}                                                                                                                       \\
                                             \hline
\textbf{HO-SIREN}                                     & 0.199 \tikzcircle[gold,fill=gold]{2pt}                & 59.199 \tikzcircle[gold,fill=gold]{2pt}     & 0.997 \tikzcircle[gold,fill=gold]{2pt}     & 41.060 \tikzcircle[silver,fill=silver]{2pt} & 0.981 \tikzcircle[silver,fill=silver]{2pt} & 0.031 \tikzcircle[bronze,fill=bronze]{2pt} & 0.794 \tikzcircle[silver,fill=silver]{2pt}                & 42.852 \tikzcircle[silver,fill=silver]{2pt} & 0.987 \tikzcircle[silver,fill=silver]{2pt} & 0.015 & 37.696 \tikzcircle[silver,fill=silver]{2pt} & 0.958 \tikzcircle[silver,fill=silver]{2pt} & 0.131 \\
\textbf{HO-Pos. Enc}                                 & 0.206 \tikzcircle[bronze,fill=bronze]{2pt}                & 44.638     & 0.991     & 40.034 \tikzcircle[bronze,fill=bronze]{2pt} & 0.978 \tikzcircle[bronze,fill=bronze]{2pt} & 0.024 \tikzcircle[silver,fill=silver]{2pt} & 0.805 \tikzcircle[bronze,fill=bronze]{2pt}                & 41.974 & 0.986 \tikzcircle[bronze,fill=bronze]{2pt} & 0.017 & 37.095 \tikzcircle[bronze,fill=bronze]{2pt} & 0.954 \tikzcircle[bronze,fill=bronze]{2pt} & 0.103 \tikzcircle[silver,fill=silver]{2pt} \\
\textbf{HO-FFN}                                      & 0.329                & 54.637 \tikzcircle[silver,fill=silver]{2pt}     & 0.997 \tikzcircle[silver,fill=silver]{2pt}     & 45.393 \tikzcircle[gold,fill=gold]{2pt} & 0.991 \tikzcircle[gold,fill=gold]{2pt} & 0.008 \tikzcircle[gold,fill=gold]{2pt} & 1.317                & 46.845 \tikzcircle[gold,fill=gold]{2pt} & 0.995 \tikzcircle[gold,fill=gold]{2pt} & 0.004 \tikzcircle[gold,fill=gold]{2pt} & 39.203 \tikzcircle[gold,fill=gold]{2pt} & 0.967 \tikzcircle[gold,fill=gold]{2pt} & 0.097 \tikzcircle[gold,fill=gold]{2pt}  \\
\hline\hline
\end{tabular}}
\label{tab: Results of Image Representation.}
\end{table*}

\subsection{Neural Tangent Kernel Perspective} \label{sec:Neural Tangent Kernel}
Our proposed HOIN method can effectively capture the high-frequency components of the signal. However, it is difficult to study this characteristic of spectral bias theoretically. The function constructed by the neural network is implicit, and its dependence on low-frequency component learning cannot be directly analyzed. Recently, some researchers have studied the learning process of neural networks through kernel function approximation  \cite{jacot2018neural}. The neural tangent kernel theory uses a first-order Taylor expansion of the model parameters $\theta$, that is:

\begin{equation}
    F_{\mathbf{\theta }}(\mathbf{x})\approx F_{\mathbf{\theta }_0}(\mathbf{x})+\left( \mathbf{\theta }-\mathbf{\theta }_0 \right) ^{\top}\nabla _{\mathbf{\theta }}F_{\mathbf{\theta }_0}(\mathbf{x}).
\end{equation}

When the width of the layer in $F_{\mathbf{\theta }}(\mathbf{x})$ is close to infinity, and the learning rate of the optimizer is close to 0, $F_{\mathbf{\theta }}(\mathbf{x})$ can converge to the kernel regression solution of the neural tangent kernel during the training process, i.e. the kernel function
\begin{equation}
\mathcal{K}_{\mathrm{NTK}}\left( \mathbf{x},\mathbf{x}^{'} \right) =\mathbb{E} _{\mathbf{\theta }\sim \mathcal{N}}\left. \langle \frac{F_{\mathbf{\theta }}(\mathbf{x})}{\partial \mathbf{\theta }},\frac{F_{\mathbf{\theta }}(\mathbf{x}^{'})}{\partial \mathbf{\theta }} \right. \rangle.
\end{equation}


By analyzing the eigenvalue distribution of the NTK kernel function, we can deeply understand the learning behavior of the neural network \cite{smola1998kernel,yuce2022structured}. When the larger eigenvalues of the kernel function are mainly concentrated in the diagonal area, the kernel function exhibits better translation invariance \cite{liu2023finer}. This structural property enables the model to learn signals more efficiently during training. In addition, when the feature value is larger, the model has a more vital high-frequency learning ability. This means that the model can respond more sensitively and learn high-frequency components in the signal.




We analyze the NTKs for various models, including activation functions like ReLU and SIREN (no encoding layers) and models that utilize encoding layers such as Pos. Enc and FFN coupled with a ReLU activation function. We visualize the NTK matrices for these models configured with Plain and high-order blocks. Unless otherwise specified, we all use three hidden layer networks to generate the NTK kernel matrix in the rest of this article.

As shown in Figure \ref{fig: ntk}(a), the kernel function of the plain block exhibits poor diagonal properties, leading to challenges in learning both low and high frequencies. Conversely, the HO block kernel matrix showcases significant diagonal eigenvalues, facilitating effective learning of low-frequency signals while concurrently capturing high-frequency signals. Furthermore, the HO block exhibits excellent diagonal properties and large feature values for SIREN, Pos. Enc, and FFN. This attribute is a crucial factor contributing to the successful characterization of high-frequency signals by these INRs. The model's diagonal width is further reduced upon integrating the high-order structure, and the eigenvalues are augmented. This advancement enhances the model's capacity to capture high-frequency information beyond the original model, facilitating nearly simultaneous learning of high-frequency and low-frequency information. Figure \ref{fig: ntk}(b) showcases the shift in the eigenvalue distribution, revealing a significant increase in the number of eigenvalues exceeding $10^1$ when utilizing HO blocks. This observation underlines the enhanced capability of the HOIN framework to learn high-frequency information effectively.

\begin{figure*}[!t]
	\centering
	\includegraphics[width=1.0\linewidth]{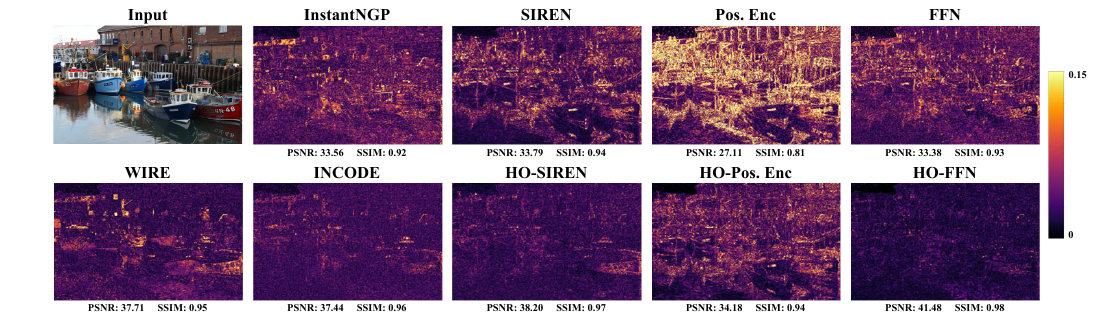}
	\caption{Visualization of Image Representation. Here, we demonstrate the representation errors of different models. The brighter areas indicate higher representation errors. HO-FFN accurately reconstructs all the detailed information of the image.}
	\label{fig: Visualization of Image Representation.}
 \Description{}
\end{figure*}

\section{Experiments}


In this Section, we conduct an extensive experimental evaluation of HOIN. Our experimental setup is detailed in Section \ref{sec:Experimental details and setup}. In Sections \ref{sec:Image Represention}, \ref{sec:Image Denoise}, \ref{sec:Image Super-Resolution}, \ref{sec:CT Reconstruction}, and \ref{sec:Image Inpainting}, we explore the application of the HOIN framework to specific image inversion tasks, including image representation, denoising, super-resolution, CT reconstruction, and image completion. Additional ablation studies and visualizations are provided in the supplementary experiments section for further insights into the effectiveness and operational mechanisms of HOIN.

\subsection{Experimental details and setup}\label{sec:Experimental details and setup}

For our experimental benchmarks, we select four models: WIRE \cite{saragadam2023wire}, SIREN \cite{sitzmann2020implicit}, Pos. Enc \cite{mildenhall2021nerf}, and FFN \cite{tancik2020fourier}. To these models, we integrate the HO block into SIREN, Pos. Enc, and FFN to identify the most effective configurations, collectively termed HOIN. We also add InstantNGP \cite{muller2022instant} and INCODE \cite{kazerouni2024incode}. The experimental scope includes tasks such as image representation (5000 epochs), image denoising (2000 epochs), image super-resolution (2000 epochs), CT image reconstruction (5000 epochs), and image completion (1000 epochs). The data for these tasks comprise randomly selected 10 images with dimensions of $1644\times2040\times3$ from the DIV2K \cite{timofte2018ntire} dataset. We set the evaluation metrics of the Peak Signal-to-Noise Ratio (PSNR) and Structural Similarity Index Measure (SSIM) \cite{wang2003multiscale}. Further experimental details are available in the supplementary materials.

\subsection{Image Represention}\label{sec:Image Represention}

Image representation can be viewed as a particular inverse problem. Its performance intuitively presents the model's improvement in spectral bias. In the experiment, we subject the images to various degrees of downsampling, namely $1\times$, $2\times$, $4\times$, and $8\times$, to cater to the diverse requirements of image representation at different downsampling rates. For the experiments involving $1\times$ and $2\times$ downsampling rates, the hidden layers are configured with 512 neurons each. Conversely, in the experiments with $4\times$ and $8\times$ downsampling, the number of neurons per hidden layer is set to 256. To maintain a fair comparison across all models, we adjust the parameter count of the InstantNGP model to align with the order of magnitude of the parameters in the other models. The outcomes of these experiments are presented in Table \ref{tab: Results of Image Representation.}.

Table \ref{tab: Results of Image Representation.} demonstrates that integrating the High-Order structure markedly enhances our model's capacity for high-frequency representation, yielding PSNR results compared to the baseline model. Notably, the HO-FFN model records the highest PSNR, registering approximately 8.1dB greater than the original FFN model. Figure \ref{fig: Visualization of Image Representation.} illustrates the error distribution during the reconstruction of $2\times$ downsampled images, revealing the HO-FFN model's near-complete reconstruction of high-frequency details, including edges and textures. For models like Pos. Enc, introducing HO structures can also effectively enhance their ability to represent high-frequency information.

\subsection{Image Denoise}\label{sec:Image Denoise}
In the image denoising experiment, to each image to assess the denoising abilities of the models, we add Gaussian noise with three noise levels, including $\sigma=10$, $\sigma=25$, and $\sigma=50$. Each model is set up with a hidden layer containing 256 neurons. The results of these experiments are detailed in the following Table \ref{tab: Results of Image Denoise}.

Table \ref{tab: Results of Image Denoise} reveals that the HO-Pos.Enc model, benefiting from the moderate acceleration provided by the HOIN framework in learning high-frequency information, exhibits superior performance across all denoising experiments. Networks utilizing the ReLU activation function have effectively learned low-frequency information, whereas the HOIN framework has demonstrated a significant advantage in acquiring high-frequency details. Furthermore, in line with previous analyses, excessive acceleration in high-frequency information learning by models such as InstantNGP and HO-FFN can result in the undesirable amalgamation of high-frequency noise with details. This conflation can detrimentally affect the outcome of denoising tasks. Detailed discussions and visualizations related to the denoising experiments are thoroughly presented in the supplementary.

\begin{table}[!t]
\caption{Image denoising results under different Gaussian noise $\sigma$.}
	\resizebox{\linewidth}{!}{
\begin{tabular}{llllllllll}
\hline\hline
\multirow{2}{*}{Methods} & \multicolumn{3}{c}{10} & \multicolumn{3}{c}{25} & \multicolumn{3}{c}{50} \\
\cmidrule(lr){2-4}\cmidrule(lr){5-7}\cmidrule(lr){8-10}
                         & PSNR $\uparrow$   & SSIM $\uparrow$  & LPIPS $\downarrow$ & PSNR $\uparrow$   & SSIM  $\uparrow$  & LPIPS $\downarrow$ & PSNR $\uparrow$    & SSIM $\uparrow$  & LPIPS $\downarrow$ \\
\hline
InstantNGP \cite{muller2022instant}               & 29.483 & 0.773 & 0.211 & 22.786 & 0.540 & 0.398 & 17.754 & 0.382 & 0.567 \\
WIRE \cite{saragadam2023wire}                   & 30.802 & 0.846 & 0.185 & 26.411 & 0.725 & 0.323 & 22.864 & 0.589 & 0.451 \\
INCODE \cite{kazerouni2024incode}                     & 30.478 & 0.811 & 0.196 & 25.113 & 0.644 & 0.357 & 21.841 & 0.509 & 0.489 \\
SIREN \cite{sitzmann2020implicit}                    & 31.140 & 0.849 & 0.169 & 26.748 \tikzcircle[bronze,fill=bronze]{2pt} & 0.734 \tikzcircle[bronze,fill=bronze]{2pt} & 0.309 & 23.407 & 0.625 & 0.426 \\
Pos. Enc \cite{mildenhall2021nerf}                & 27.271 & 0.736 & 0.281 & 25.756 & 0.685 & 0.306 \tikzcircle[bronze,fill=bronze]{2pt} & 23.459 \tikzcircle[silver,fill=silver]{2pt} & 0.643 \tikzcircle[silver,fill=silver]{2pt} & 0.414 \tikzcircle[silver,fill=silver]{2pt} \\
FFN \cite{tancik2020fourier}               & 31.131 & 0.852 & 0.180 & 26.335 & 0.716 & 0.317 & 22.805 & 0.583 & 0.453 \\
\hline
\multicolumn{10}{c}{\textbf{Ours}}                                                                           \\
\hline
\textbf{HO-SIREN}                 & 32.338 \tikzcircle[silver,fill=silver]{2pt} & 0.896 \tikzcircle[silver,fill=silver]{2pt} & 0.144 \tikzcircle[silver,fill=silver]{2pt} & 27.043 \tikzcircle[silver,fill=silver]{2pt} & 0.766 \tikzcircle[silver,fill=silver]{2pt} & 0.300 \tikzcircle[silver,fill=silver]{2pt} & 23.451 \tikzcircle[bronze,fill=bronze]{2pt} & 0.642 \tikzcircle[bronze,fill=bronze]{2pt} & 0.425 \tikzcircle[bronze,fill=bronze]{2pt} \\
\textbf{HO-Pos. Enc}             & 32.452 \tikzcircle[gold,fill=gold]{2pt} & 0.910 \tikzcircle[gold,fill=gold]{2pt} & 0.123 \tikzcircle[gold,fill=gold]{2pt} & 27.561 \tikzcircle[gold,fill=gold]{2pt} & 0.801 \tikzcircle[gold,fill=gold]{2pt} & 0.264 \tikzcircle[gold,fill=gold]{2pt} & 23.858 \tikzcircle[gold,fill=gold]{2pt} & 0.677 \tikzcircle[gold,fill=gold]{2pt} & 0.390 \tikzcircle[gold,fill=gold]{2pt} \\
\textbf{HO-FFN}            & 32.057 \tikzcircle[bronze,fill=bronze]{2pt} & 0.882 \tikzcircle[bronze,fill=bronze]{2pt} & 0.152 \tikzcircle[bronze,fill=bronze]{2pt} & 26.596 & 0.726 & 0.310 & 22.990 & 0.574 & 0.447 \\
\hline\hline
\end{tabular}}
\label{tab: Results of Image Denoise}
\end{table}

\subsection{Image Super-Resolution}\label{sec:Image Super-Resolution}
In our image super-resolution experiment, we initially downsample the original images by 2, 4, 6, and 8 factors. These downsampled images are used in the training phase to leverage the inherent interpolation capabilities of INR. Subsequently, in the testing phase, we aim to restore them to their original dimensions. The comprehensive results of these experiments are meticulously documented in Table \ref{tab: Results of Image Super-Resolution}, showcasing the effectiveness of our approach in enhancing image super-resolution.

As shown in Table \ref{tab: Results of Image Super-Resolution}, the HOIN framework markedly enhances the performance across all evaluated models. Notably, HO-SIREN exhibits outstanding PSNR and SSIM metrics across most super-resolution tasks. In contrast, due to its intrinsic methodology, the InstantNGP model, which relies on hash table indexes for reconstruction, proves less adept for pixel-aligned super-resolution tasks. Additional visualization results and detailed analyses are available in the supplementary.

\begin{table}[!t]
\caption{Results of Image Super-Resolution.}
	\resizebox{\linewidth}{!}{
\begin{tabular}{lllllllll}
\hline\hline
\multicolumn{1}{c}{\multirow{2}{*}{Methods}} & \multicolumn{2}{c}{$\times2$} & \multicolumn{2}{c}{$\times4$} & \multicolumn{2}{c}{$\times6$} & \multicolumn{2}{c}{$\times8$} \\
\cmidrule(lr){2-3}\cmidrule(lr){4-5}\cmidrule(lr){6-7}\cmidrule(lr){8-9}
\multicolumn{1}{c}{}                         & PSNR $\uparrow$        & SSIM $\uparrow$         & PSNR $\uparrow$        & SSIM $\uparrow$        & PSNR $\uparrow$        & SSIM $\uparrow$        & PSNR $\uparrow$        & SSIM $\uparrow$        \\
\hline
InstantNGP \cite{muller2022instant}                                   & 19.74       & 0.400       & 16.42       & 0.203       & 16.44       & 0.212       & 16.35       & 0.246       \\
WIRE \cite{saragadam2023wire}                                         & 31.50       & 0.846       & 29.09 \tikzcircle[bronze,fill=bronze]{2pt}       & 0.786       & 26.77 \tikzcircle[bronze,fill=bronze]{2pt}       & 0.717       & 24.44       & 0.686       \\
INCODE \cite{kazerouni2024incode}                                       & 31.94       & 0.853       & 28.97       & 0.812       & 26.47       & 0.759       & 24.95 \tikzcircle[bronze,fill=bronze]{2pt}       & 0.694       \\
SIREN \cite{sitzmann2020implicit}                                        & 31.61       & 0.851       & 28.26       & 0.803       & 26.23       & 0.736       & 24.18       & 0.715       \\
Pos. Enc \cite{mildenhall2021nerf}                                    & 30.39       & 0.805       & 24.26       & 0.745       & 24.36       & 0.718       & 23.28       & 0.713       \\
FFN \cite{tancik2020fourier}                                         & 31.38       & 0.856       & 27.93       & 0.795       & 26.16       & 0.783 \tikzcircle[bronze,fill=bronze]{2pt}       & 24.49       & 0.728 \tikzcircle[bronze,fill=bronze]{2pt}       \\
\hline
\multicolumn{9}{c}{\textbf{Ours}}                                                                                                                                     \\
\hline
\textbf{HO-SIREN}                                     & 33.03 \tikzcircle[silver,fill=silver]{2pt}       & 0.898 \tikzcircle[silver,fill=silver]{2pt}       & 29.61 \tikzcircle[gold,fill=gold]{2pt}       & 0.854 \tikzcircle[gold,fill=gold]{2pt}       & 27.53 \tikzcircle[gold,fill=gold]{2pt}       & 0.815 \tikzcircle[gold,fill=gold]{2pt}       & 25.69 \tikzcircle[gold,fill=gold]{2pt}       & 0.771 \tikzcircle[gold,fill=gold]{2pt}       \\
\textbf{HO-Pos. Enc}                                 & 32.47 \tikzcircle[bronze,fill=bronze]{2pt}       & 0.876 \tikzcircle[bronze,fill=bronze]{2pt}       & 28.91       & 0.824 \tikzcircle[bronze,fill=bronze]{2pt}       & 26.43       & 0.762       & 24.78       & 0.720       \\
\textbf{HO-FFN}                                       & 33.10 \tikzcircle[gold,fill=gold]{2pt}      & 0.898 \tikzcircle[gold,fill=gold]{2pt}       & 29.30 \tikzcircle[silver,fill=silver]{2pt}       & 0.839 \tikzcircle[silver,fill=silver]{2pt}       & 27.30 \tikzcircle[silver,fill=silver]{2pt}       & 0.798 \tikzcircle[silver,fill=silver]{2pt}       & 25.44 \tikzcircle[silver,fill=silver]{2pt}       & 0.759 \tikzcircle[silver,fill=silver]{2pt}       \\
\hline\hline
\end{tabular}}
\label{tab: Results of Image Super-Resolution}
\end{table}

\begin{table}[!t]
\caption{CT reconstruction results from different angles.}
	\resizebox{\linewidth}{!}{
\begin{tabular}{lllllllll}
\hline\hline
\multicolumn{1}{c}{\multirow{2}{*}{Methods}} & \multicolumn{2}{c}{50} & \multicolumn{2}{c}{100} & \multicolumn{2}{c}{200} & \multicolumn{2}{c}{300} \\
\cmidrule(lr){2-3}\cmidrule(lr){4-5}\cmidrule(lr){6-7}\cmidrule(lr){8-9}
\multicolumn{1}{c}{}                         & PSNR $\uparrow$        & SSIM $\uparrow$         & PSNR $\uparrow$        & SSIM $\uparrow$        & PSNR $\uparrow$        & SSIM $\uparrow$        & PSNR $\uparrow$        & SSIM $\uparrow$        \\
\hline
InstantNGP \cite{muller2022instant}                                   & 17.56      & 0.569     & 18.65      & 0.662      & 20.59      & 0.743      & 22.21      & 0.795      \\
WIRE \cite{saragadam2023wire}                                         & 21.93      & 0.648     & 26.28      & 0.799      & 29.01      & 0.814      & 29.20      & 0.818      \\
INCODE \cite{kazerouni2024incode}                                       & 22.76      & 0.674     & 26.63      & 0.701      & 31.16      & 0.819      & 32.22 \tikzcircle[bronze,fill=bronze]{2pt}      & 0.861      \\
SIREN \cite{sitzmann2020implicit}                                        & 22.96      & 0.714     & 26.96      & 0.745      & 27.97      & 0.822      & 30.32      & 0.847      \\
Pos. Enc \cite{mildenhall2021nerf}                                    & 22.72      & 0.734     & 23.78      & 0.784      & 24.20      & 0.809      & 24.30      & 0.801      \\
FFN \cite{tancik2020fourier}                                          & 26.03      & 0.779 \tikzcircle[bronze,fill=bronze]{2pt}     & 30.17 \tikzcircle[bronze,fill=bronze]{2pt}      & 0.898      & 31.79 \tikzcircle[bronze,fill=bronze]{2pt}      & 0.925      & 32.03      & 0.936      \\
\hline
\multicolumn{9}{c}{\textbf{Ours}}                                                                                                                            \\
\hline
\textbf{HO-SIREN}                                     & 28.02 \tikzcircle[gold,fill=gold]{2pt}      & 0.866 \tikzcircle[silver,fill=silver]{2pt}     & 32.12 \tikzcircle[gold,fill=gold]{2pt}      & 0.932 \tikzcircle[silver,fill=silver]{2pt}      & 34.41 \tikzcircle[silver,fill=silver]{2pt}      & 0.963 \tikzcircle[gold,fill=gold]{2pt}      & 34.82 \tikzcircle[silver,fill=silver]{2pt}      & 0.968 \tikzcircle[gold,fill=gold]{2pt}      \\
\textbf{HO-Pos. Enc}                                 & 26.94 \tikzcircle[silver,fill=silver]{2pt}      & 0.906 \tikzcircle[gold,fill=gold]{2pt}     & 28.39      & 0.933 \tikzcircle[gold,fill=gold]{2pt}      & 28.79      & 0.944 \tikzcircle[bronze,fill=bronze]{2pt}      & 29.40      & 0.949 \tikzcircle[bronze,fill=bronze]{2pt}      \\
\textbf{HO-FFN}                                       & 26.85 \tikzcircle[bronze,fill=bronze]{2pt}      & 0.769     & 30.82 \tikzcircle[silver,fill=silver]{2pt}      & 0.912 \tikzcircle[bronze,fill=bronze]{2pt}      & 34.90 \tikzcircle[gold,fill=gold]{2pt}      & 0.962 \tikzcircle[silver,fill=silver]{2pt}      & 34.83 \tikzcircle[gold,fill=gold]{2pt}      & 0.961 \tikzcircle[silver,fill=silver]{2pt}      \\
\hline\hline
\end{tabular}}

\label{tab: Results of CT Reconstruction}
\end{table}

\begin{figure*}[!t]
	\centering
	\includegraphics[width=1\linewidth]{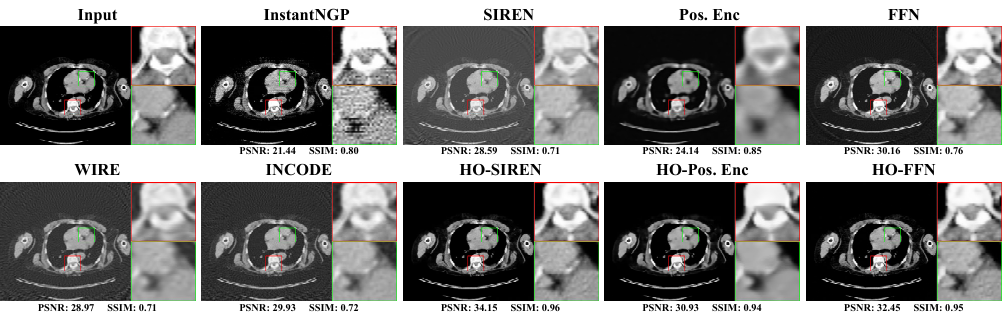}
	\caption{Visualization of Computed Tomography Reconstruction. Here, we demonstrate various methods for CT-based reconstruction of $256 \times 256$ images at 100 angles. HO-FFN maintains the best reconstruction results.}
	\label{fig: Visualization of Computed Tomography Reconstruction}
 \Description{}
\end{figure*}

\begin{figure}[!t]
	\centering
	\includegraphics[width=1\linewidth]{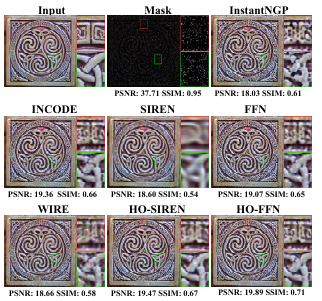}
	\caption{Visualization of Image Inpainting. Here, we only use 10\% of the original image's pixels for reconstruction. HO-SIREN effectively reconstructs detail levels such as texture edges.}
	\label{fig: Visualization of Image Inpainting.}
 \Description{}
\end{figure}

\subsection{CT Reconstruction}\label{sec:CT Reconstruction}
Our CT image reconstruction experiment utilizes 10 CT lung images from the publicly accessible lung nodule analysis dataset on Kaggle \cite{clark2013cancer}. To assess the efficacy of our model in CT reconstruction tasks, these images are downsampled to a resolution of $256\times256$. The experiment involves measuring reconstruction at four angles: 50, 100, 200, and 300. The findings of these experiments are comprehensively detailed in Table \ref{tab: Results of CT Reconstruction}.

CT reconstruction involves creating computational images from sensor measurements, with sparse CT reconstruction tackling the challenge of producing accurate images from limited measurement data. This challenge is primarily due to the difficulty in reconstructing images with scarce data. HO block substantially improves the quality of the reconstruction results by efficiently capturing high-frequency components throughout the reconstruction process. As detailed in Table \ref{tab: Results of CT Reconstruction}, the HOIN model outperforms others in all measurement scenarios, showcasing its superior performance. Figure \ref{fig: Visualization of Computed Tomography Reconstruction} illustrates that the HO-SIREN model is particularly adept at reconstructing images' texture and contour details. In comparison, the SIREN model, much like the WIRE and INCODE models, is prone to artifacts, whereas InstantNGP struggles with significant pixel loss issues.

\subsection{Image Inpainting}\label{sec:Image Inpainting}
In the image inpainting experiment, we select an image of a Celtic spiral knot with a resolution of $572 \times 582 \times 3$. The mask applied in this experiment is generated randomly, obscuring approximately $10\%$ of the image's pixel area. The architectural configuration for this experiment is aligned with that used in the image representation task, with the findings presented in Figure \ref{fig: Visualization of Image Inpainting.}.

Compared to existing SOTA methods based on INR, the HO-FFN model demonstrates considerable superiority in image inpainting tasks, particularly in accurately rendering details. To corroborate the efficacy of our approach in image completion tasks, additional relevant experiments are included in the supplementary for further examination.

\section{Conclusion}
In this paper, we propose \textbf{High-Order Implicit Neural Representations
for Inverse Problems (HOIN)}, an innovative framework for addressing inverse problems. By integrating high-order interaction blocks into INR, HOIN substantially enlarges the functional space of INR to enhance the model's capacity to capture high-frequency information. The NTK matrix associated with HOIN features notable diagonal and translational invariance, offering robust theoretical backing to mitigate spectral bias. Unlike alternative approaches, HOIN is adept at diminishing noise interference and swiftly and efficiently resolving inverse problems. Through comprehensive experiments, HOIN has been shown to outperform other models utilizing INR for inverse problem-solving and has also excelled in representation tasks.

\bibliographystyle{ACM-Reference-Format}
\bibliography{ArXiv}

\appendix
The supplementary material provides detailed theoretical analyses, proofs (see section \ref{sec: Theoretical Analysis}), and numerous additional experiments (see section \ref{Additional experiments and details} and \ref{Theoretical experimental verification}).

\section{Theoretical Analysis}\label{sec: Theoretical Analysis}
In this section, we detail the theory and proofs regarding the expression spaces and high-order derivatives of HOIN, as referenced in the main paper. 
\subsection{Expression Ability Exploration }\label{sec:sup_Expression ability}

In INR frameworks, the dimension of a network's functional space is a crucial metric for assessing the network's capacity for expression \cite{bu2021quadratic}. The architecture of the network is denoted by $\mathcal{D} =\left\{ D_1,...,D_l \right\}$, where $D_l$ indicates the number of neurons in the $l$-th layer. Any given activation function block can be decomposed into a series of polynomial functions with leading degree $r$ through Taylor approximation. This process helps understand how activation functions and network configurations influence INR models' functional capacity and expressiveness. 

For the network architecture $D_l$ with an activation function of leading degree $r$, we represent the leading functional space of the neural network as $\mathcal{F} _{\mathcal{D},r}$. The leading functional variants of plain Block, residual Block, and HO block can be defined as the $Zariski$ $closure$ \cite{kileel2019expressive} of their leading functional space, i.e., $\mathcal{V} _{\mathcal{D},r}^{P}$,$\mathcal{V} _{\mathcal{D},r}^{R}$ and $\mathcal{V} _{\mathcal{D},r}^{HO}$ (similar to the ones presented in \cite{bu2021quadratic,kileel2019expressive}) as follows

\begin{definition}
Suppose there is a neural network $\mathcal{D} =\left\{ D_1,...,D_l \right\}$ whose leading space satisfies the following condition:
\begin{equation}
    \mathcal{F}_{\mathcal{D},r}=\mathrm{Sym}_{r^{l-1}}\left( \mathbb{R} ^{D_0} \right) ^{D_l}.
\end{equation}
For a filling functional variety, its leading functional variety satisfies:
\begin{equation}
    \mathcal{V}_{\boldsymbol{d},r}=\overline{\mathcal{F}_{\mathcal{D} ,r}}=\overline{\mathrm{Sym}_{r^{l-1}}\left( \mathbb{R} ^{D_0} \right) ^{D_l}}.
\end{equation}
\label{def:leading}
\end{definition}

Thus, the function space or the varieties of the network do not have to completely occupy the ambient space of homogeneous polynomials. Instead, we only need to consider the space of homogeneous polynomials whose leading degrees are contained as being adequately filled.

\begin{proposition}
\label{pp:sup_single_layer_QRes}
For a single-layer network $\mathcal{D} =(D_l)$ utilizing a linearly activated $(r=1)$ High-Order (HO) Block, the network has a filling functional space of degree 2. That is, its leading functional space satisfies the following criteria:
\begin{equation}
    \mathcal{F} _{\mathcal{D} ,1}^{HO}=\mathrm{Sym}_2(\mathbb{R} ^{D_l})^{D_l}.
\end{equation}
\end{proposition}

\begin{proof}
We can relate the linear HO Block to a quadratic polynomial regression. Consider a HO Blcok:
    \begin{equation}\label{eq:4}
        \begin{aligned}
	\mathbf{z}_l&=\left( \mathbf{J}+\mathbf{C}_l\mathbf{z}_{l-1} \right) \odot \mathbf{z}_{l-1}\\
	&=\mathbf{z}_{l-1}+\left( \mathbf{C}_l\mathbf{z}_{l-1} \right) \odot \mathbf{z}_{l-1},\\
\end{aligned}
    \end{equation}
Where $\mathbf{z}_l$ represents a primary linear term with $\mathrm{Sym}_1(\mathbb{R} ^{D_l})^{D_l}$, and $\left( \mathbf{C}_l\mathbf{z}_{l-1} \right) \odot \mathbf{z}_{l-1}$ is a quadratic term with $\mathrm{Sym}_2(\mathbb{R} ^{D_l})^{D_l}.$ This quadratic term outlines the primary functional space of the HO Block. By Definition \ref{def:leading}, a single-layer HO Block encompasses a filling functional space of degree 2.
\end{proof}

\begin{theorem}
\label{Theorem:sup_efficiency}
For an activation function with leading degree $r \geq 1$ and network architecture $\mathcal{D} =\left\{ D_1,...,D_l \right\}$, the leading functional variety of Plain Block, $\mathcal{V} _{\mathcal{D} ,r}^{P}$, HO Block, $\mathcal{V} _{\mathcal{D} ,r}^{HO}$, and Residual Block, $\mathcal{V} _{\mathcal{D} ,r}^{R}$, satisfy:
\begin{equation}
    \mathcal{V} _{\mathcal{D} ,r}^{HO} = \mathcal{V} _{\mathcal{D} ,2r}^{R} = \mathcal{V} _{\mathcal{D} ,2r}^{P}.
\end{equation}
\end{theorem}

\begin{proof}
This can be proven by discussing the equivalence of functional space for every Block using Proposition \ref{pp:sup_single_layer_QRes}. For the $i$-th layer in the HO Block, $i=1, 2, ..., l$, before applying nonlinear activation, it has $\mathcal{V} _{(D_i),1}^{HO}=\mathrm{sym}_2(\mathbb{R} ^{D_i})^{D_i}=\mathcal{V} _{(D_i),2}^{P}=\mathcal{V} _{(D_i),2}^{R}$ (since a single-layer Blcok with polynomial activation of degree $2$ has a filling functional space of degree $2$). This proves the case for $r=1$. For nonlinear activations of leading degree $r$, applying the activation function to the space $\mathcal{V} _{(D_i),1}^{HO}$, we obtain: $\mathcal{V} _{(D_i),r}^{HO}=\left( \mathcal{V} _{(D_i),1}^{HO} \right) ^{\otimes r}=\left( \mathcal{V} _{(D_i),2}^{P} \right) ^{\otimes r}=\mathcal{V} _{(D_i),2r}^{P}
$, where $\otimes$ denotes Kronecker product. Since the relation applies to each layer, thus we have $\mathcal{V} _{\mathcal{D},r}^{HO}=\mathcal{V} _{\mathcal{D},2r}^{P}=\mathcal{V} _{\mathcal{D},2r}^{R}.$
\end{proof}

\begin{figure*}[!t]
    \centering
    \subfloat[Reconstruction error.]
    {\includegraphics[width=0.68\textwidth]{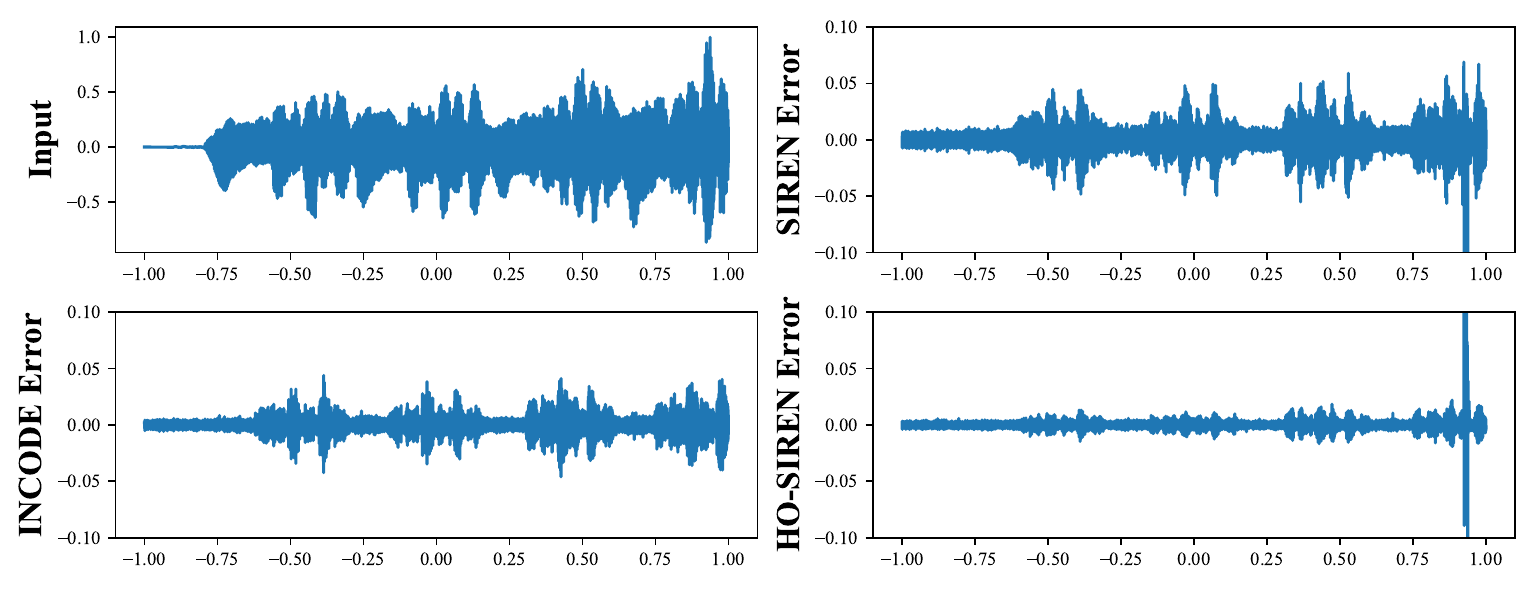}}
    \subfloat[Reconstructed PSNR.]
    {\includegraphics[width=0.32\textwidth]{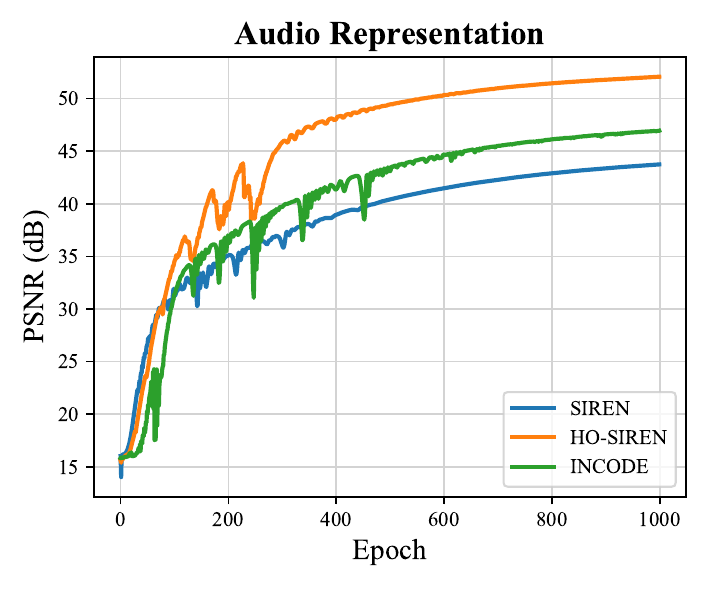}}
\caption{Results of the audio representation.  (a) Reconstruction error. (b) Reconstructed PSNR. HO-SIREN excels by minimizing reconstruction errors and demonstrating rapid convergence.}
\Description{}
\label{fig: sup_audio}
  \vspace{-3mm}
\end{figure*}

\begin{figure*}[!t]
	\centering
 \vspace{-4mm}
 	\includegraphics[width=1\textwidth]{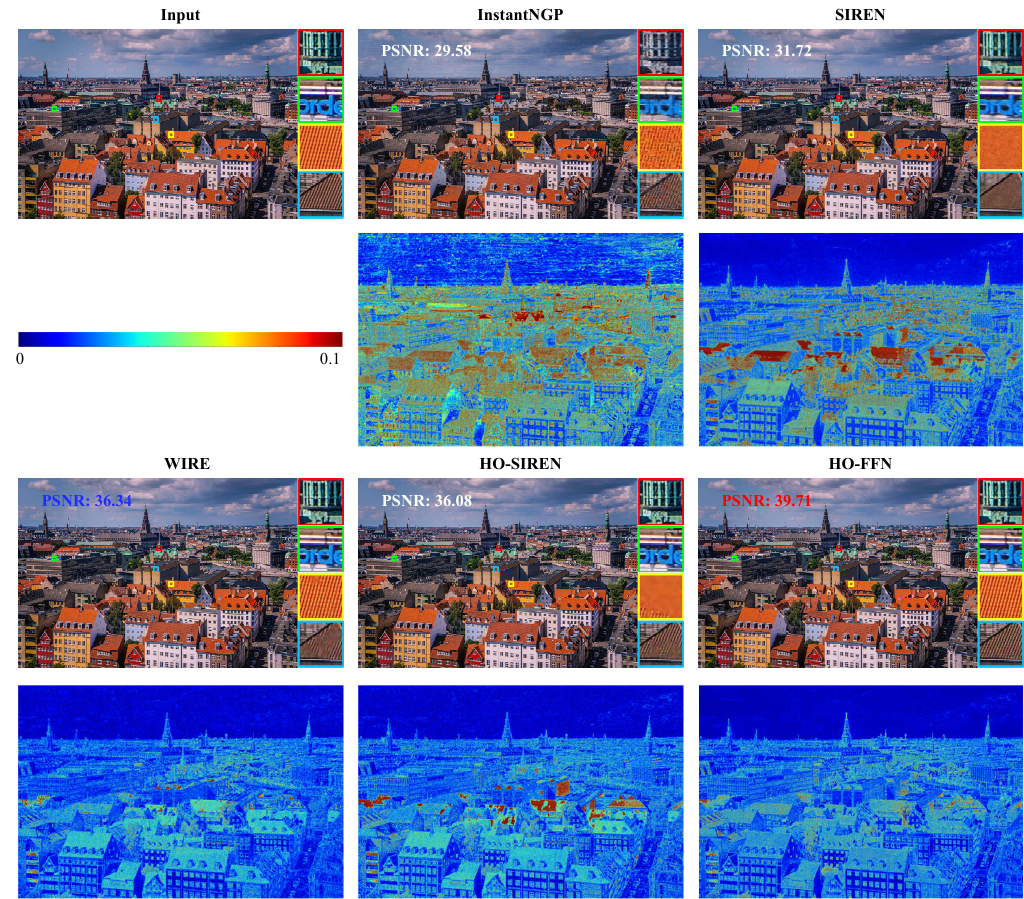}
   \vspace{-5mm}
	\caption{Image representation results. The second row is an error map. The darker the red color, the higher the error. For PSNR, red is the best, and blue is the second best. HO-FFN is the best representation result.}
	\label{fig:sup_large_image}
 \Description{}
\vspace{-4mm}
\end{figure*}

Theorem \ref{Theorem:sup_efficiency} posits that within an identical network architecture, HO block with a leading degree of $r$ and plain and residual block with a leading degree of $2r$ exhibit the same variation in homogeneous polynomial functions. This implies that the HO block can support a diversity of leading functions, specifically $(2r)^{l-1}$ homogeneous polynomials, in contrast to neural networks of the same structure and activation function, which are limited to $r^{l-1}$. Consequently, the HO block possesses a more expansive expression space, representing a broader range of frequency signal components.

Besides, we give the frequency decay rates for different blocks as follows:
\begin{theorem} 
\label{Theorem:sup_frequency decay}
For a single-layer MLP, the number of neurons is $d$. For frequency $k$, the frequency decay rate of the Plain and Residual Block is $k^{-d}$, while on the contrary, the frequency decay rate of the HO Block is $k^{{-d}/2}$.
\end{theorem}

\begin{proof}
    For Plain and Residual Block,  the frequency attenuation rate and proof are shown in \cite{belfer2021spectral}. For HO Blcok, the frequency attenuation rate and proof are shown in \cite{choraria2022spectral}.
\end{proof}

\subsection{Derivative Analysis}\label{sec:sup_High-Order Derivative}
The first and second derivatives are pivotal for encapsulating rich high-frequency information in signal processing. Spectral bias often arises because high-order derivatives gravitate towards zero during the learning process with plain and residual blocks, leading to a substantial loss of high-frequency details. The HO block is ingeniously designed to mitigate this issue. 

Consider a HO Block in \ref{eq:4}, the first derivative of the HO block $\nabla \mathbf{z}_l$ is
as follows

\begin{equation}
\begin{aligned}
	\nabla \mathbf{z}_l&=\frac{\partial \mathbf{z}_l}{\partial \mathbf{z}_{l-1}}\\
	&=\partial \mathbf{z}_{l-1}\left( \left( \mathbf{J}+\mathbf{C}_l\mathbf{z}_{l-1} \right) \odot \mathbf{z}_{l-1} \right)\\
	&=\left( \partial \mathbf{z}_{l-1}\left( \mathbf{J}+\mathbf{C}_l\mathbf{z}_{l-1} \right) \right) \odot \mathbf{z}_{l-1}+\left( \mathbf{J}+\mathbf{C}_l\mathbf{z}_{l-1} \right) \odot \left( \partial \mathbf{z}_{l-1}\left( \mathbf{z}_{l-1} \right) \right)\\
	&=\mathbf{C}_l\odot \mathbf{z}_{l-1}+\mathbf{C}_l\mathbf{z}_{l-1}+\mathbf{J}.\\
\end{aligned}
\end{equation}

The second derivative of the HO Block $\Delta \mathbf{z}_l$ is

\begin{equation}
\begin{aligned}
	\Delta \mathbf{z}_l&=\nabla \left( \nabla \mathbf{z}_l \right)\\
	&=\frac{\partial \left( \nabla \mathbf{z}_l \right)}{\partial \mathbf{z}_{l-1}}\\
	&=\partial \mathbf{z}_{l-1}\left( \mathbf{C}_l\odot \mathbf{z}_{l-1}+\mathbf{C}_l\mathbf{z}_{l-1}+\mathbf{J} \right)\\
	&=\partial \mathbf{z}_{l-1}\left( \mathbf{C}_l\odot \mathbf{z}_{l-1} \right) +\partial \mathbf{z}_{l-1}\left( \mathbf{C}_l\mathbf{z}_{l-1}+\mathbf{J} \right)\\
	&=\left( \partial \mathbf{z}_{l-1}\left( \mathbf{C}_l \right) \right) \odot \mathbf{z}_{l-1}+\mathbf{C}_l\odot \left( \partial \mathbf{z}_{l-1}\left( \mathbf{z}_{l-1} \right) \right) +\mathbf{C}_l\\
	&=\mathbf{C}_l+\mathbf{C}_l\\
	&=2\mathbf{C}_l.\\
\end{aligned}
\end{equation}

The second derivative of the HO block remains constant, a property that significantly enhances its ability to capture detailed signal information. This characteristic of the HO block is instrumental in efficiently accelerating the resolution of inverse problems by preserving and leveraging high-frequency details often lost in traditional processing blocks.

\section{Additional experiments and details}\label{Additional experiments and details}

In this section, we broaden the scope of our experiments to provide a more thorough comparison between our method, HOIN, and the current state-of-the-art (SOTA) method. We show that the inherent simplicity of HOIN leads to enhanced performance, particularly in terms of expressiveness and the ability to tackle inverse problems, compared to the corresponding SOTA method. These results underscore our approach's effectiveness in extending the INR network's capabilities and enhancing its applicability across various domains. We now include additional visualizations that distinctly highlight the advantages of our method.

\subsection{Experimental details}

All implementations utilize MLP networks with three hidden layers. Our experiments use PyTorch on an Nvidia RTX 3080 Ti GPU with 12GB of RAM. We employ the Adam optimizer, complemented by a learning rate scheduler that decreases the learning rate by 0.1 upon completion of each epoch. Details on specific datasets and architectures are provided for the corresponding tasks.
\subsection{Signal Representation}

\subsubsection{\textbf{Audio}}\quad 
\newline
\textbf{Data:} We use the first 7 seconds of Bach’s Cello Suite No. 1: Prelude \cite{kazerouni2024incode}, with a sampling rate of 44100 Hz as our example for the audio representation task.
\newline
\textbf{Architecture:} All models use three hidden layers with 256 neurons per hidden layer. We set the first layer $w_0 = 10000$ for SIREN, HO-SIREN, INCODE. Each model is trained for a total of 1000 epochs.
\newline
\textbf{Analysis:} We present the audio representation visualization results and their corresponding error plots in Figure \ref{fig: sup_audio}. These visualizations are crucial for illustrating the strengths of our approach. Regarding sound playback quality, SIREN tends to introduce a distinct squeak-like sound accompanying the main audio. With INCODE, certain moments experience annoying noise, as the error chart indicates. However, HO-SIREN significantly reduces noise interference and outperforms the other methods.

\subsubsection{\textbf{Image}}\quad 
\newline
\textbf{Data:} In the main paper and supplementary material experiments, we select one of the larger nature images on the Internet with a size of $3\times4844\times3219$.
\newline
\textbf{Architecture:} For all models except InstantNGP, we use three hidden layers, each with 512 neurons. We set the first layer's frequency parameter $w_0=30$ for SIREN, HO-SIREN, and INCODE. For the Wire, we set the scaling parameter $s_0=20$. Each model is trained for a total of 1000 epochs.
\newline
\textbf{Analysis:} Figure \ref{fig:sup_large_image} displays the visualization results of large-size images. In terms of reconstruction quality, InstantNGP incorrectly represents some colors. Because of its inherent spectral bias, SIREN struggles with reconstructing high-frequency details, such as the top of the house, with distinct light and dark variations. HO-SIREN and HO-FFN excel in capturing high-frequency details and consistently maintain the best overall quality, 3 dB better than the current top-performing WIRE (SOTA) model.

\begin{table}[h]
\caption{Best 3 scores in each metric are marked with gold \tikzcircle[gold,fill=gold]{2pt}, silver \tikzcircle[silver,fill=silver]{2pt} and bronze \tikzcircle[bronze,fill=bronze]{2pt}.}
\begin{tabular}{l|ll}
\hline \hline
Methods          & Thai  & Lucy   \\
\hline
INCODE     & 0.9879      & 0.9951 \tikzcircle[gold,fill=gold]{2pt} \\
WIRE       & 0.9903 \tikzcircle[bronze,fill=bronze]{2pt}      & 0.9718 \\
SIREN      & 0.9758      & 0.9885 \\
Pos.Enc    & 0.9872      & 0.9927 \\
\hline
\multicolumn{3}{c}{\textbf{Ours} }         \\
\hline
\textbf{HO-SIREN}   & 0.9935 \tikzcircle[gold,fill=gold]{2pt}      & 0.9948 \tikzcircle[silver,fill=silver]{2pt} \\
\textbf{HO-Pos.Enc} & 0.9918 \tikzcircle[silver,fill=silver]{2pt}      & 0.9945 \tikzcircle[bronze,fill=bronze]{2pt}  \\
\hline\hline
\end{tabular}
\label{tab: sup_oc}
\vspace{-4mm}
\end{table}

\begin{figure*}[!t]
	\centering
 	\includegraphics[width=1\textwidth]{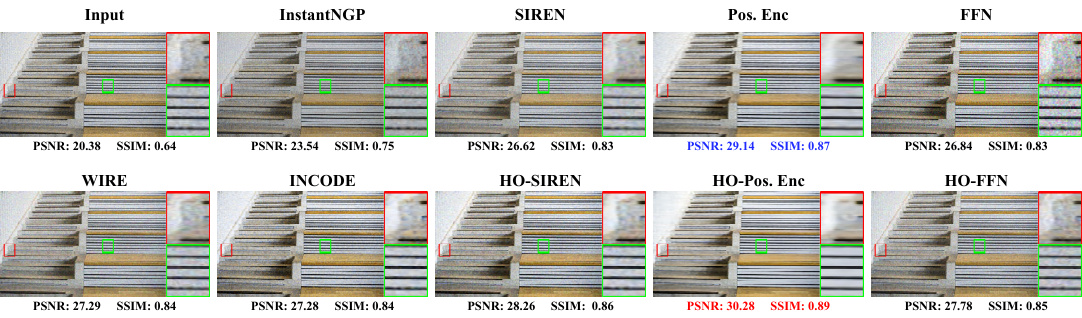}
	\caption{The result of image denoising. HO-Pos. enc maintains the best PSNR and SSIM.}
	\label{fig:sup_denoise}
 \Description{}
\end{figure*}

\begin{figure*}[!t]
	\centering
 	\includegraphics[width=1\textwidth]{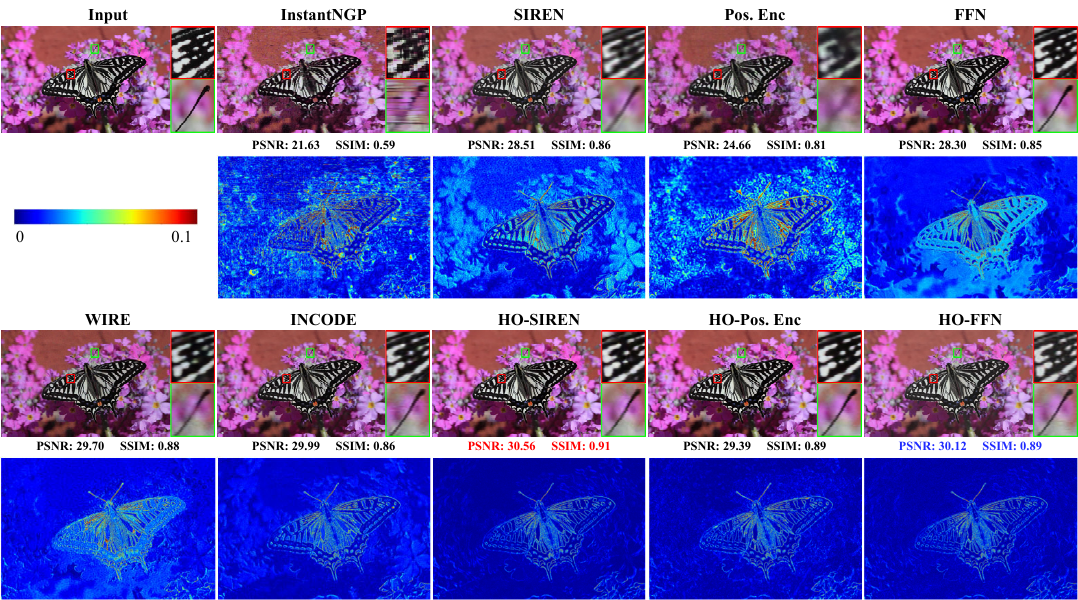}
	\caption{The result of image super-resolution. The second line is the error map. HO-SIREN maintains the best PSNR and SSIM and can accurately characterize textures, edges, and other detailed information.}
	\label{fig:sup_sup}
 \Description{}
\end{figure*}

\begin{figure*}[!t]
	\centering
 	\includegraphics[width=1\textwidth]{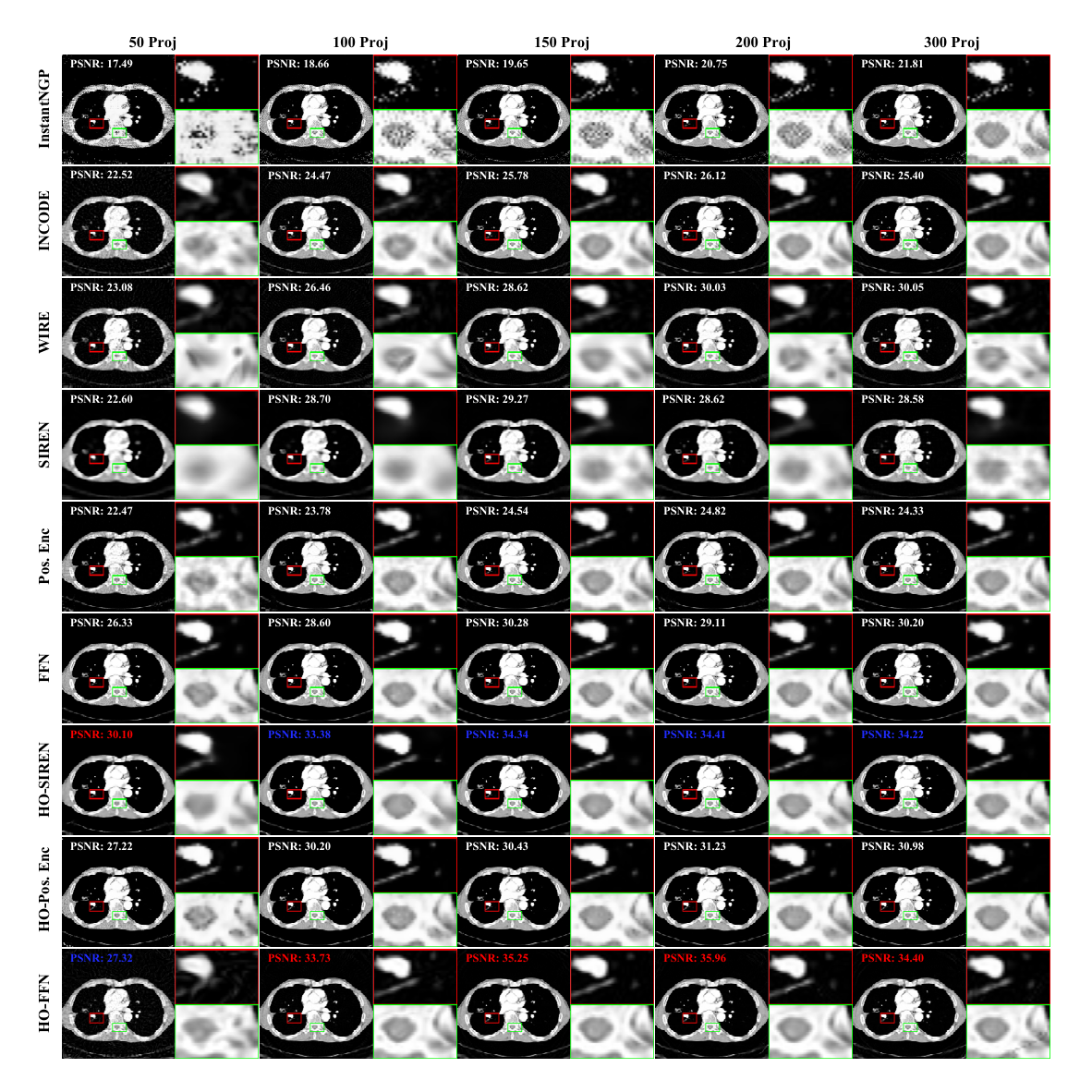}
	\caption{Reconstruction results for different projects. HO-FFN always maintains the best reconstruction results, keeping the highest PSNR and SSIM.}
	\label{fig:sup_CT}
 \Description{}
\end{figure*}

\begin{table*}[!thbp]
\caption{The result of image inpainting.}
\begin{tabular}{lllllllll}
\hline\hline
\multicolumn{1}{c}{\multirow{2}{*}{Methods}} & \multicolumn{2}{c}{20\%} & \multicolumn{2}{c}{40\%} & \multicolumn{2}{c}{60\%} & \multicolumn{2}{c}{80\%} \\
\cmidrule(lr){2-3}\cmidrule(lr){4-5}\cmidrule(lr){6-7}\cmidrule(lr){8-9}
\multicolumn{1}{c}{}                         & PSNR $\uparrow$        & SSIM $\uparrow$         & PSNR $\uparrow$        & SSIM $\uparrow$        & PSNR $\uparrow$        & SSIM $\uparrow$        & PSNR $\uparrow$        & SSIM $\uparrow$        \\
\hline
InstantNGP                                   & 20.719      & 0.734      & 23.400      & 0.829      & 25.263      & 0.871      & 26.747      & 0.896      \\
WIRE                                         & 21.118      & 0.710      & 24.015      & 0.827      & 26.073      & 0.885 \tikzcircle[bronze,fill=bronze]{2pt}      & 26.677      & 0.895      \\
INCODE                                       & 22.097 \tikzcircle[bronze,fill=bronze]{2pt}      & 0.781 \tikzcircle[silver,fill=silver]{2pt}      & 24.796      & 0.859      & 25.922      & 0.877      & 27.279 \tikzcircle[bronze,fill=bronze]{2pt}      & 0.903 \tikzcircle[bronze,fill=bronze]{2pt}      \\
SIREN                                        & 21.318      & 0.719      & 24.044      & 0.828      & 25.495      & 0.865      & 26.380      & 0.885      \\
FFN                                          & 21.862      & 0.773      & 25.127 \tikzcircle[silver,fill=silver]{2pt}      & 0.868 \tikzcircle[silver,fill=silver]{2pt}      & 27.617 \tikzcircle[silver,fill=silver]{2pt}      & 0.912 \tikzcircle[silver,fill=silver]{2pt}      & 29.754 \tikzcircle[silver,fill=silver]{2pt}      & 0.938 \tikzcircle[silver,fill=silver]{2pt}      \\
\hline
\multicolumn{9}{c}{Ours}                                                                                                                                 \\
\hline
HO-SIREN                                     & 22.119 \tikzcircle[silver,fill=silver]{2pt}      & 0.777 \tikzcircle[bronze,fill=bronze]{2pt}      & 24.961 \tikzcircle[bronze,fill=bronze]{2pt}      & 0.866 \tikzcircle[bronze,fill=bronze]{2pt}      & 26.376 \tikzcircle[bronze,fill=bronze]{2pt}      & 0.881      & 27.279      & 0.894      \\
HO-FFN                                       & 22.357 \tikzcircle[gold,fill=gold]{2pt}      & 0.802 \tikzcircle[gold,fill=gold]{2pt}      & 25.962 \tikzcircle[gold,fill=gold]{2pt}      & 0.896 \tikzcircle[gold,fill=gold]{2pt}      & 28.764 \tikzcircle[gold,fill=gold]{2pt}      & 0.936 \tikzcircle[gold,fill=gold]{2pt}      & 31.573 \tikzcircle[gold,fill=gold]{2pt}      & 0.958 \tikzcircle[gold,fill=gold]{2pt}      \\
\hline\hline
\end{tabular}
\label{tab: sup_inpainting}
\end{table*}

\subsubsection{\textbf{Occupancy Volume}}\quad 
\newline
\textbf{Data:} We use the Lucy and Thai Statue datasets from the Stanford 3D Scanning Repository and follow the WIRE strategy \cite{saragadam2023wire}. We create an occupancy volume through point sampling on a $512 \times 512 \times 512$ grid, assigning values of 1 to voxels within the object and 0 to voxels outside.
\newline
\textbf{Architecture:} Our network and training configuration is similar to the image representation task, with the difference that INR now maps 3D coordinates to signed distance function (SDF) values. Each model is trained for a total of 100 epochs.
\newline
\textbf{Analysis:} The results in Table \ref{tab: sup_oc} demonstrate the effectiveness of HO-SIREN as a formidable option for occupancy representation tasks. HO-SIREN significantly improves representation by effectively utilizing the HO Block to capture complex interactions between features. This capability is especially evident in enhancing high-frequency information while maintaining excellent capture of low-frequency details. Our approach achieves higher "Intersection over Union" (IOU) values, significantly enhancing object detail and scene complexity and rendering more accurately than existing methods.

\subsection{Inverse Problems}
\subsubsection{\textbf{Image denoising}}\quad 
\newline
\textbf{Data:} In the experiments in the main paper and supplementary materials, we employ an image from DIV2K dataset \cite{timofte2018ntire},  downsampled by a factor of $1/4$ from $1152 \times2040 \times3$ to $288 \times 510 \times 3$. we add Gaussian noise with three noise levels, including $\sigma=10$, $\sigma=25$, and $\sigma=50$.
\newline
\textbf{Architecture:} The setup for the denoising experiment closely mirrors that of the image characterization experiment, with the modification that the neurons in each model were adjusted to 256. Throughout the training process, we monitored the Peak signal-to-noise ratio (PSNR) of both the noisy and clean images, considering the peak PSNR of the clean image as the final result of the reconstruction. Each model is trained for a total of 2000 epochs.
\newline
\textbf{Analysis:} Experimental results for the three noise scales are presented in the main paper. Here, we visualize the experimental results of $\sigma=25$. As shown in Figure \ref{fig:sup_denoise}, HO-Pos.Enc substantially enhances the fidelity of noisy images, achieving a 9.9 dB improvement in PSNR and a 0.22 increase in Structural Similarity Index (SSIM). Compared to the INCODE and SIREN methods, HO-Pos.Enc more effectively reduces noise artifacts while delicately preserving image details. Furthermore, our method surpasses the Pos. Enc. method in terms of SSIM by 0.02.

\subsubsection{\textbf{Image super resolution}}\quad 
\newline
\textbf{Data:} We adopt an image from the DIV2K dataset \cite{timofte2018ntire} and downsampled the image with the size of $1356 \times 2040 \times 3$ by factors of $1/2$, $1/4$, $1/6$, and $1/8$.
\newline
\textbf{Architecture:} We maintain the same architectural and training settings as the image representation task. By employing a downsampled image during training, we exploit the interpolation capabilities of INRs to reconstruct an image of its original size in the test. Each model is trained for a total of 500 epochs.
\newline
\textbf{Analysis:} The experimental results for four upsampling factors are shown in the main paper. Here, we visualize the experimental results for one map with an upsampling factor of 4. The application of INRs as interpolators holds significant promise in super-resolution, leveraging inherent biases within INRs that can be utilized to enhance performance in such tasks. As depicted in Figure \ref{fig:sup_sup}, HO-SIREN and HO-FFN consistently achieve superior PSNR and SSIM values across various super-resolution scales, surpassing competing methods. Additionally, HO-SIREN excels in reconstructing detailed elements like high-quality and transparent textures, avoiding the background artifacts commonly associated with SIREN.

\subsubsection{\textbf{CT reconstruction}}\quad 
\newline
\textbf{Data:} Our CT image reconstruction experiment utilizes 10 CT lung images from the publicly accessible lung nodule analysis dataset on Kaggle \cite{clark2013cancer}. To assess the efficacy of our model in CT reconstruction tasks, these images are downsampled to a resolution of $256\times256$. The experiment measures reconstruction at four angles and projects: 50, 100, 200, and 300.
\newline
\textbf{Architecture:} We maintain the same architectural and training settings as the image representation task. We generate a sinogram according to the projection level using the radon transform. The model predicts a reconstructed CT image. Subsequently, we calculate the radon transform for the generated output and compute the loss function between these sinograms to guide the model toward generating CT images with reduced artifacts. Each model is trained for a total of 5000 epochs.
\newline
\textbf{Analysis:} CT reconstruction involves creating computational images from sensor measurements. In sparse CT reconstruction, the challenge is generating accurate images using only a limited subset of the available measurements, complicating the imaging process. As shown in Figure \ref{fig:sup_CT}, HOIN addresses this challenge by effectively integrating higher-order interactions between features using the HO Block. The HO-FFN leverages 100 measurements to produce a sharp reconstruction with crisp details, achieving a notable improvement of 5.12 dB over the standard FFN, thus distinguishing itself in performance. In contrast, SIREN, similar to WIRE and INCODE, exhibited artifacts. This underscores the robustness of HO-FFN in managing noisy and undersampled inverse problems, demonstrating its potential as a promising solution for constrained image reconstruction, where it effectively balances image fidelity with noise reduction.

\subsubsection{\textbf{Inpainting}}\quad 
\newline
\textbf{Data:} We utilize Celtic spiral knots image with a $572\times 582\times 3$ resolution. The sampling masks are generated randomly, with an average of 20\%, 40\%, 60\%, and 80\% of pixels being sampled.
\newline
\textbf{Architecture:} We use the same structure as the image representation. Each model is trained for a total of 500 epochs.
\newline
\textbf{Analysis:} Image inpainting poses a significant challenge, as the task requires the model to predict entire pixel values based on only a small portion of trained pixel data. The experimental results are shown in Table \ref{tab: sup_inpainting}. The high capacity of INRs offers a unique advantage in addressing this inverse problem. The strong prior embedded within the INR function space facilitates applications such as repairs from finite observations, where the model leverages its learned representation to predict and fill in missing values. As seen in other tasks, HO-FFN excels in capturing complex features, particularly edges, which allows it to outperform other methods that often yield ambiguous results.

\begin{figure}[!h]
    \centering
    \includegraphics[width=1\linewidth]{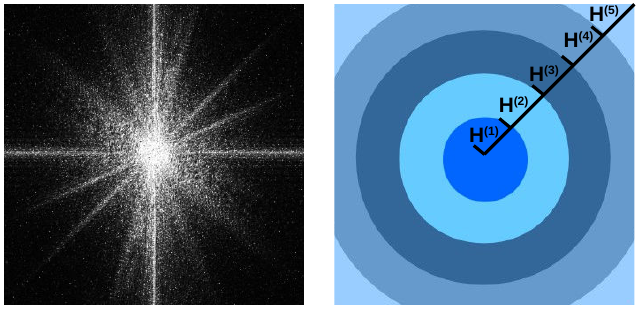}
    \caption{\textbf{Frequency-band correspondence metric.} The left image shows an example of correspondence map $H$, which is computed according to Eq. (\ref{eq:fbc}). We divide the correspondence map into $N$ subgroups corresponding to $N$ non-overlapping frequency bands. Since the correspondence map is symmetrical around the center, we group it uniformly according to the distance between its elements and center, as illustrated by the right image when $N=5$. Different colors represent different subgroups. We compute the mean correspondence for each band to transform the 2D map into the 1D one.}
    \label{exp_fig_architecture}
\label{fig:fbc}
\Description{}
\end{figure}

\begin{figure*}[!t]
	\centering
 	\includegraphics[width=1\textwidth]{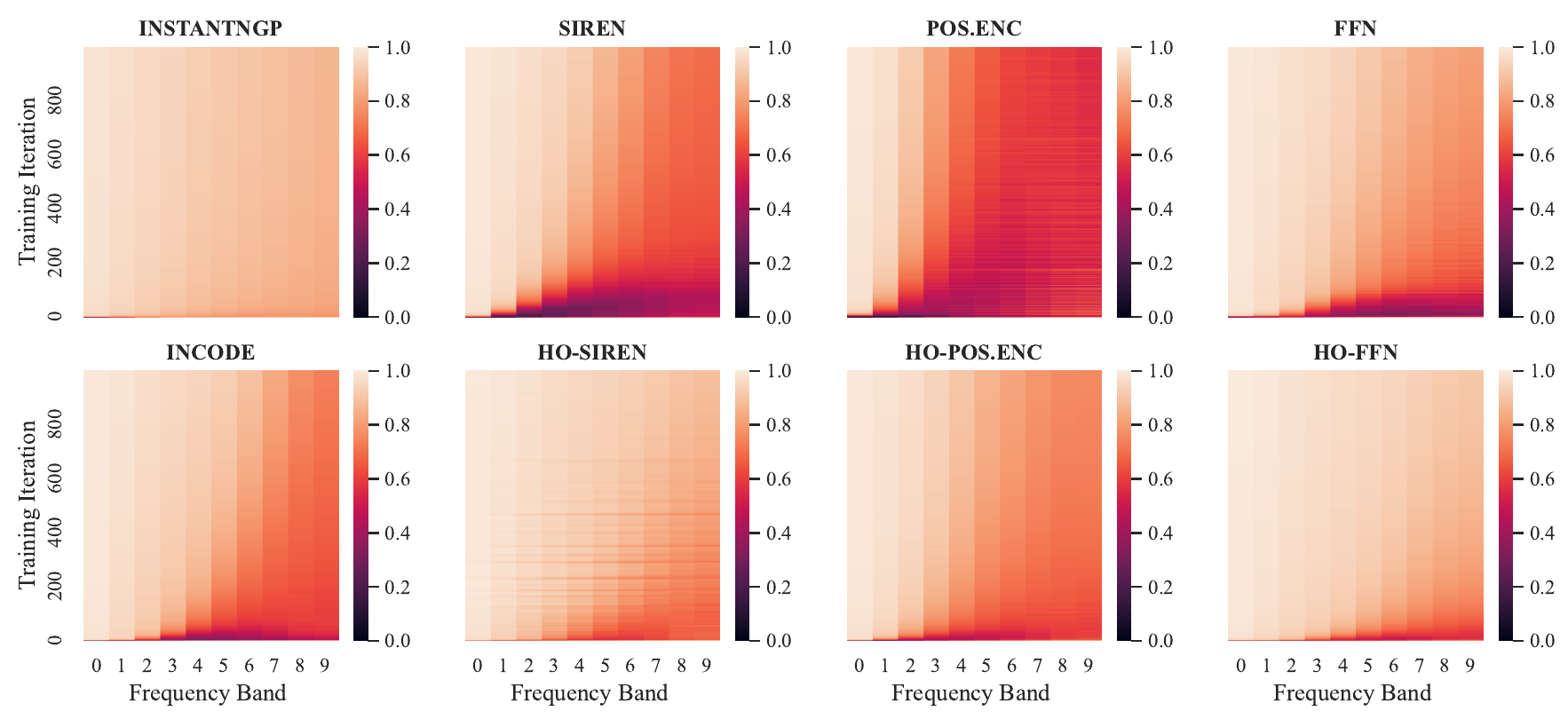}
	\caption{Comparison of learning speeds at different frequencies. The target image is transformed into 10 frequency bands through the Fourier transform (x-axis, 0 represents the lowest frequency band), and we compare the learned components with the proper amplitude. On the color chart scale, 1 represents a perfect approximation. HO block can effectively alleviate spectral bias.}
	\label{fig:sup_sb}
 \Description{}
\end{figure*}

\begin{figure*}[!t]
    \centering
    \subfloat[Noise Image PSNR]
    {\includegraphics[width=0.5\textwidth]{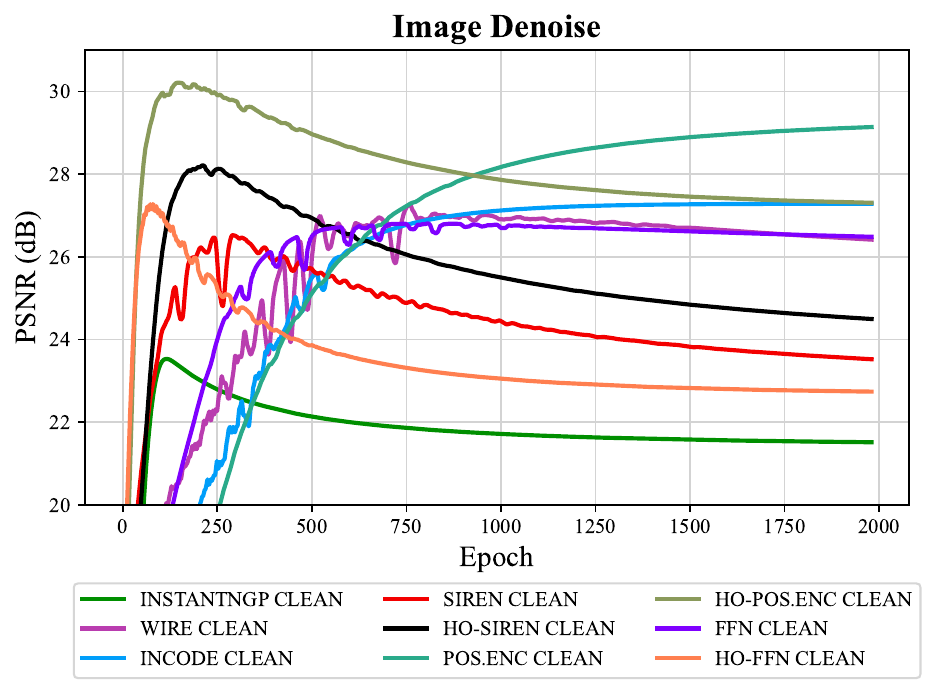}}
    \subfloat[Clean Image PSNR]
    {\includegraphics[width=0.5\textwidth]{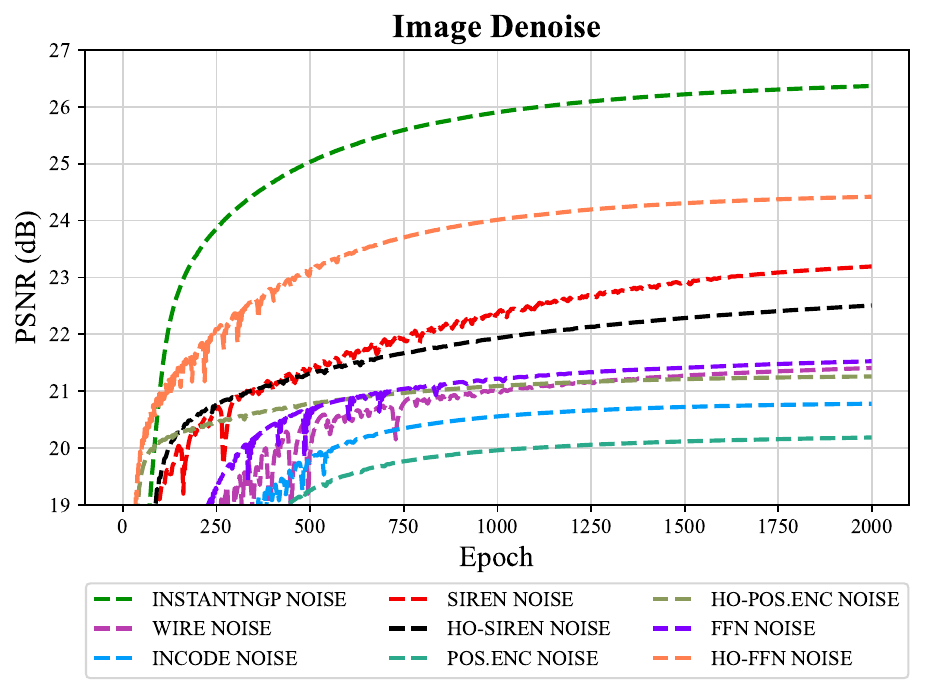}}
\caption{The PSNR learning curves for noisy and clean images are in Figure \ref{fig:sup_denoise}. HO-Pos. Enc and HO-SIREN effectively speed up the learning process compared to Pos. Enc and SIREN are reflected in both high noisy PSNR and clean PSNR.}
\Description{}
\label{fig: clean_psnr}
  \vspace{-3mm}
\end{figure*}

\begin{figure*}[!t]
	\centering
 	\includegraphics[width=1\textwidth]{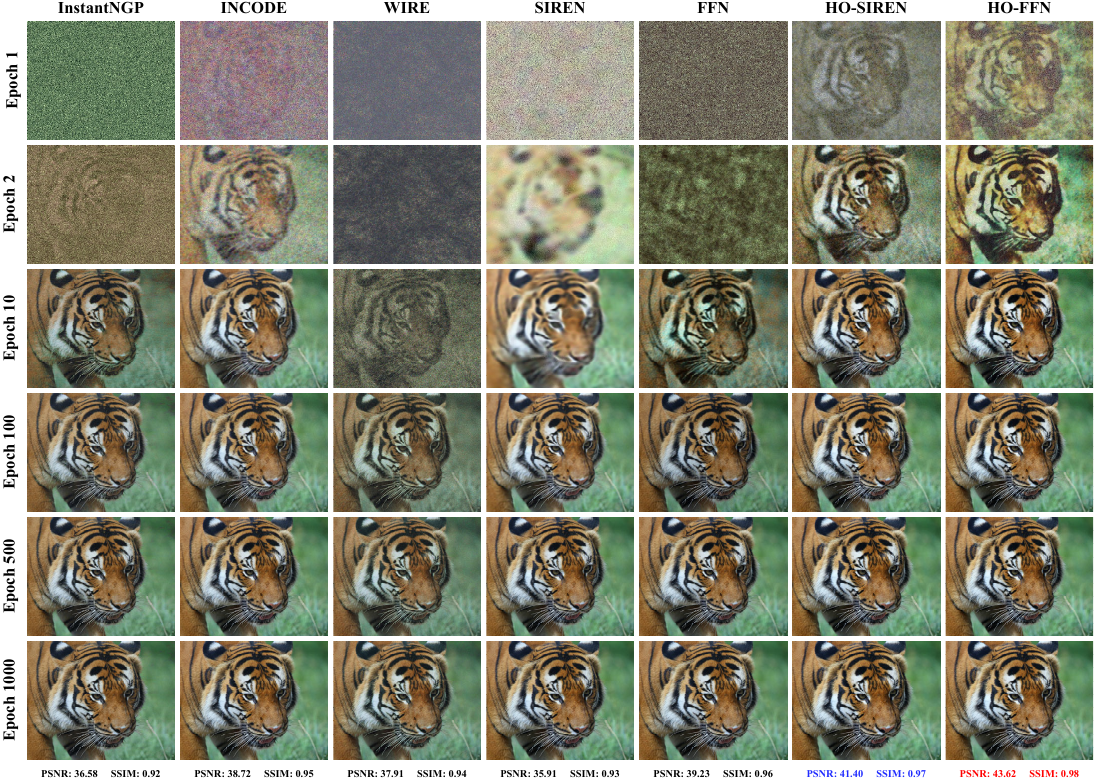}
	\caption{The results across different epochs. After just two epochs, HO-SIREN and HO-FFN can capture the outline and color features of the tiger, performing even better than InstantNGP. Throughout subsequent iterations, HO-FFN consistently maintains strong visual reconstruction results.}
	\label{fig:sup_dif_epoch}
 \Description{}
\end{figure*}

\section{Theoretical experimental verification}\label{Theoretical experimental verification}

\subsection{Convergence rate comparison}

\subsubsection{\textbf{Experimental settings}}\quad
\newline
\label{sec_fbc}
We use the band correspondence metric in \cite{shi2022measuring} to check the input-output correspondence across multiple bands in the frequency domain Multiple bands For this metric, let $\{\theta^{(1)},\dots,\theta^{(T)}\}$ denote the trajectory of $T$ steps of gradient descent in the parameter space and let $\{F_{\theta^{(1)}},\dots, F_{\theta^{(T)}}\}$ denote the corresponding trajectory in the output space. We propose to analyze the Fourier spectrum of the output images $F_{\theta^{(t)},t{=}1,\dots,T}$ to show the convergence dynamics of different frequency components of the target image.
The Fourier spectrum of the output image $F_{\theta^{(t)}}$ is obtained by the Fourier transform $\mathscr{F}$, denoted as $\mathscr{F}\{F_{\theta^{(t)}}\}$ for step $t$. We similarly compute the Fourier transform for the target image $G$, denoted as $\mathscr{F}\{G_0\}$. We then compute an element-wise correspondence between both transforms as:
\begin{equation}
H_{\theta^{(t)}}=\frac{\mathscr{F}\{F_{\theta^{(t)}}\}}{\mathscr{F}\{G_0\}}.
\label{eq:fbc}
\end{equation}
Intuitively, $H_{\theta^{(t)}}$ denotes to what extent any deep image prior at step $t$ corresponds with image $G_0$ in the frequency domain; the closer the values are to 1, the higher the correspondence.
As we are interested in the spectral bias of the deep image prior, we divide the correspondence map into $N$ subgroups corresponding to $N$ non-overlapping frequency bands. Since the correspondence map is symmetrical around the center, we group it uniformly according to the distance between its elements and its center, as illustrated in Figure \ref{fig:fbc}. To transform the 2D map to the 1D one, we compute the mean correspondence for each band, denoted as $\bar{H}_{\theta^{(t)}}^{(n)}$, with $n{=}1, \dots, N$. The value of $\bar{H}_{\theta^{(t)}}^{(n)}$ indicates the convergence dynamics of different frequency components of a target image.

\subsubsection{\textbf{Experimental results}}\quad
\newline
This section analyzes various models' spectral bias and convergence speed, including InstantNGP, INCODE, SIREN, and Pos. Enc, FFN, HO-SIREN, HO-Pos. Enc, and HO-FFN. We use the configuration in section \ref{sec_fbc}. We conducted an image representation experiment where, as described in the main paper, the image is divided into 10 frequency bands from low to high. The metric is the ratio of the learned frequency band values to the true image for each epoch.

As depicted in Figure \ref{fig:sup_sb}, the darker the red color, the less information is learned in that frequency band. Models like SIREN, Pos. Enc. and FFN struggle to learn high-frequency information in the early stages of training. However, introducing the HO Block significantly enhances the model's perception of high-frequency information. Among these, HO-SIREN and HO-FFN consistently achieve the best results. But for InstantNGP, the approach involves dividing the image into countless grids and simultaneously fitting the values at these grid points, effectively learning low-frequency and high-frequency information simultaneously. This method contributes to its rapid fitting capabilities.

We also visualize the results across different epochs. After just two epochs, HO-SIREN and HO-FFN could capture the outline and color features of the tiger, performing even better than InstantNGP. Throughout subsequent iterations, HO-FFN consistently maintained strong visual reconstruction results.

\subsection{Spectral bias in inverse tasks}
In this section, we address the application of mitigating spectral bias in inverse problems. Properly reducing spectral bias to enhance the perception of high-frequency information can accelerate the resolution of inverse problems. However, excessive acceleration might lead to premature coupling of high-frequency noise with the signal's high-frequency information, complicating the resolution of inverse problems. Using image denoising as an example, we visualize the PSNR learning curves for noisy and clean images in Figure \ref{fig: clean_psnr}. HO-Pos. Enc and HO-SIREN effectively speed up the learning process compared to Pos. Enc and SIREN are reflected in both high noisy PSNR and clean PSNR. Conversely, InstantNGP and HO-FFN, due to their overly aggressive mitigation of spectral bias, experience coupling of noise and signal at high frequencies, which is detrimental for image denoising tasks.

For the representation task, where excessive mitigation of spectral bias is not a concern, HO-FFN emerges as the top-performing model. In denoising tasks, HO-Pos.Enc strikes the best balance and proves to be the most effective model. For tasks involving super-resolution, CT reconstruction, and inpainting, both HO-SIREN and HO-FFN stand out as the best models.

\end{document}